\documentclass[nohyperref]{article}

\usepackage{etoolbox}

\newtoggle{arxiv}

\toggletrue{arxiv}

\usepackage{microtype}
\usepackage{graphicx}
\usepackage{subfigure}
\usepackage{booktabs} 

\usepackage[numbers]{natbib}
\usepackage{jcd}
\bluehyperref
\usepackage{gary_macros}
\usepackage{priv-int-macros}
\usemedgeometry
\usepackage{cleveref}
\usepackage{algorithm}
\usepackage{algorithmic}

\newtheorem{rem}{Remark}

\usepackage{amsmath}
\usepackage{amssymb}
\usepackage{mathtools}
\usepackage{amsthm}
\usepackage{thmtools,thm-restate}

\usepackage[textsize=tiny]{todonotes}

\newcommand{\footremember}[2]{%
   \footnote{#2}
    \newcounter{#1}
    \setcounter{#1}{\value{footnote}}%
}
\newcommand{\footrecall}[1]{%
    \footnotemark[\value{#1}]%
}

\renewcommand{\citet}{\cite}

\renewcommand{\S}{\mathbb{S}}
\newcommand{\xd}{\mc{X}} 
\title{Private optimization in the interpolation regime:\\ faster rates and hardness results}

\author{%
  Hilal Asi\footremember{authorship}{Equal contribution, author order alphabetical }\footremember{eedept}{Electrical Engineering Department, Stanford University} \\
  \texttt{asi@stanford.edu} 
  \and
  Karan Chadha\footrecall{authorship} \footrecall{eedept}\\
  \texttt{knchadha@stanford.edu} 
    \and
   Gary Cheng\footrecall{authorship}
   \footrecall{eedept} \\
  \texttt{chenggar@stanford.edu} 
  \and
  John Duchi\footrecall{eedept}
   \footnote{Statistics Department, Stanford University}\\
  \texttt{jduchi@stanford.edu} 
}

\begin{document}

\maketitle

\begin{abstract}
In non-private stochastic convex optimization, stochastic gradient methods converge much faster on interpolation problems---problems where there exists a solution that simultaneously minimizes all of the sample losses---than on non-interpolating ones;
we show that generally similar improvements are impossible in the private setting. However, when the functions exhibit quadratic growth around the optimum, we show (near) exponential improvements in the private sample complexity. In particular, we propose an adaptive algorithm that improves the sample complexity to achieve expected error $\alpha$ from $\frac{d}{\diffp \sqrt{\alpha}}$ to $\frac{1}{\alpha^\rho} + \frac{d}{\diffp} \log\paren{\frac{1}{\alpha}}$ for any fixed $\rho >0$, while retaining the standard minimax-optimal sample complexity for non-interpolation problems. We prove a lower bound that shows the dimension-dependent term is tight. Furthermore, we provide a superefficiency result which demonstrates the necessity of the polynomial term for adaptive algorithms: any algorithm that has a polylogarithmic sample complexity for interpolation problems cannot achieve the minimax-optimal rates for the family of non-interpolation problems.
\end{abstract}

\section{Introduction} \label{sec:intro}




We study differentially private stochastic convex optimization (DP-SCO), where given a dataset $\mc{S} = S_1^n \simiid P$ we wish to solve 
\begin{equation}
  \label{eqn:objective}
  \begin{split}
    \minimize ~ & f(x) = \E_P[F(x;\statrv)]
    = \int_\statdomain F(x; \statval) dP(\statval) \\
    \subjectto ~ & x \in \mc{X},
  \end{split}
\end{equation}
while guaranteeing differential privacy. In problem~\eqref{eqn:objective}, $\xd \subset \R^d$ is the parameter space, $\statdomain$ is a sample space, and $\{F(\cdot;s):
s\in\S \}$ is a collection of convex losses.
We study the interpolation setting, where there exists a solution that simultaneously minimizes all of the sample losses.

Interpolation problems are ubiquitous in machine learning applications: for example, least squares problems with consistent solutions~\cite{StrohmerVe09,NeedellWaSr14}, and problems with over-parametrized models where a perfect predictor exists~\cite{MaBaBe18,BelkinHsMi18,BelkinRaTs19}. This has led to a great deal of work on the advantages and implications of interpolation~\cite{SrebroSrTe10,CotterShSrSr11,BelkinHsMi18,BelkinRaTs19}.

 For non-private SCO, interpolation problems allow significant improvements in convergence rates over generic problems~\cite{SrebroSrTe10,CotterShSrSr11,MaBaBe18,VaswaniBaSc19,WoodworthSr21}. For general convex functions, \citet{SrebroSrTe10} develop algorithms that obtain $O(\frac{1}{n})$ sub-optimality, improving over the minimax-optimal rate $O(\frac{1}{\sqrt{n}})$ for non-interpolation problems. Even more dramatic improvements are possible when the functions exhibit growth around the minimizer, as~\citet{VaswaniBaSc19} show that SGD achieves exponential rates in this setting compared to polynomial rates without interpolation. \cite{AsiDu19siopt, AsiChChDu20, ChadhaChDu22} extend these fast convergence results to model-based optimization methods.

Despite the recent progress and increased interest in interpolation problems, in the private setting they remain poorly understood. In spite of the substantial progress in characterizing the tight convergence guarantees for a variety of settings in DP optimization~\cite{BassilySmTh14,BassilyFeTaTh19,FeldmanKoTa20,AsiFeKoTa21,AsiLeDu21}, we have little understanding of private optimization in the growing class of interpolation problems.

Given (i) the importance of differential privacy and interpolation problems in modern machine learning, (ii) the (often) paralyzingly slow rates of private optimization algorithms, and (iii) the faster rates possible for non-private interpolation problems, the interpolation setting provides a reasonable opportunity for significant speedups in the private setting. This motivates the following two questions: first, is it possible to improve the rates for DP-SCO in the interpolation regime? And, what are the optimal rates?





\subsection{Our contributions}
We answer both questions. In particular, we show that

\begin{enumerate}
	\item \textbf{No improvements in general} (\Cref{sec:no-growth}): our first result is a hardness result demonstrating that the rates cannot be improved for DP-SCO in the interpolation regime with general convex functions. More precisely, we prove a lower bound of $\Omega(\frac{d}{n \diffp})$ on the excess loss for pure differentially private algorithms. This shows that existing algorithms achieve optimal private rates for this setting. 
	\item \textbf{Faster rates with growth} (\Cref{sec:growth}): when the functions exhibit quadratic growth around the minimizer, that is,  $f(x) - f(x^\star) \ge  \lambda \norms{x-x^\star}_2^2$ for some $\lambda >0$, we propose an algorithm that achieves near-exponentially small excess loss, improving over the polynomial rates in the non-interpolation setting. Specifically, we show that the sample complexity to achieve expected excess loss $\alpha>0$ is $O(\frac{1}{\alpha^\rho} + \frac{d}{\diffp} \log\paren{\frac{1}{\alpha}})$ for pure DP and $O(\frac{1}{\alpha^\rho} + \frac{\sqrt{d \log(1/\delta)}}{\diffp} \log\paren{\frac{1}{\alpha}})$ for \ed-DP, for any fixed $\rho>0$. This improves over the sample complexity for non-interpolation problems with growth which is $O(\frac{1}{\alpha} + \frac{d}{\diffp \sqrt{\alpha}})$. 
	We also present new algorithms that improve the rates for interpolation problems with the weaker $\kappa$-growth assumption~\cite{AsiLeDu21} for $\growth > 2$ where we achieve excess loss $O( ( \frac{1}{\sqrt{n}} + \frac{d}{n \diffp} )^{\frac{\kappa}{\kappa - 2}} )$, compared to the previous bound $O( ( \frac{1}{\sqrt{n}} + \frac{d}{n \diffp} )^{\frac{\kappa}{\kappa - 1}} )$ without interpolation.
	\item \textbf{Adaptivity to interpolation} (\Cref{sec:growth-adap}):
	While these improvements for the interpolation regime are important, practitioners using these methods in practice cannot identify whether the dataset they are working with is an interpolating one or not.
	Thus, it is crucial that these algorithms do not fail when given a non-interpolating dataset. 
	We show that our algorithms are adaptive to interpolation, obtaining these better rates for interpolation while simultaneously retaining the standard minimax optimal rates for non-interpolation problems.
	\item \textbf{Tightness} (\Cref{sec:super}): finally, we provide a lower bound and a super-efficiency result that demonstrate the (near) tightness of our upper bounds showing sample complexity $\Omega(\frac{d}{\diffp} \log\paren{\frac{1}{\alpha}}) $ is necessary for interpolation problems with pure DP. Moreover, our super-efficiency result shows that the polynomial dependence on $1/\alpha$ in the sample complexity is necessary for adaptive algorithms: any algorithm that has a polylogarithmic sample complexity for interpolation problems cannot achieve minimax-optimal rates for non-interpolation problems.
\end{enumerate}

\subsection{Related work}

Over the past decade, a lot of works ~\cite{ChaudhuriMoSa11,DuchiJoWa13_focs,SmithTh13lasso, BassilySmTh14, 
  AbadiChGoMcMiTaZh16, BassilyFeTaTh19, FeldmanKoTa20,AsiFeKoTa21,AsiDuFaJaTa21,BassilyFeGuTa20} have studied the problem of private convex optimization. 
  \citet{ChaudhuriMoSa11} and \cite{BassilySmTh14} study the closely related problem of differentially private empirical risk minimization (DP-ERM) where the goal is to minimize the empirical loss, and obtain (minimax) optimal rates of $d/n\diffp$ for pure DP and ${\sqrt{d \log(1/\delta)}}/{n \diffp}$ for $(\diffp,\delta)$-DP. Recently, more papers have moved beyond DP-ERM to privately minimizing the population loss (DP-SCO)~\cite{BassilyFeTaTh19,FeldmanKoTa20,AsiFeKoTa21,AsiDuFaJaTa21,BassilyGuNa21,AsiLeDu21}.  ~\citet{BassilyFeTaTh19} was the first paper to obtain the optimal rate $1/\sqrt{n}+ {\sqrt{d \log(1/\delta)}}/{n \diffp}$ for \ed-DP, and subsequent papers develop more efficient algorithms that achieve the same rates~\cite{FeldmanKoTa20,BassilyFeGuTa20}. Moreover, other papers study DP-SCO under different settings including non-Euclidean geometry~\cite{AsiFeKoTa21,AsiDuFaJaTa21}, heavy-tailed data~\cite{WangXiDeXu20}, and functions with growth~\cite{AsiLeDu21}. However, to the best of our knowledge, there has not been any work in private optimization that studies the problem in the interpolation regime.

On the other hand, the optimization literature has witnessed numerous papers on the interpolation regime~\cite{SrebroSrTe10,CotterShSrSr11,MaBaBe18,VaswaniBaSc19,LiuBe20,WoodworthSr21}. \citet{SrebroSrTe10} propose algorithms that roughly achieve the rate $1/n + \sqrt{f\opt/n}$ for smooth and convex functions where $f\opt = \min_{x \in \xd} f(x)$. In the interpolation regime with $f\opt=0$, this result obtains loss $1/n$ improving over the standard $1/\sqrt{n}$ rate for non-interpolation problems. Moreover, \citet{VaswaniBaSc19} studied the interpolation regime for functions with growth and show that SGD enjoys linear convergence (exponential rates). More recently, several papers investigated and developed acceleration-based algorithms in the interpolation regime~\cite{LiuBe20,WoodworthSr21}.



\section{Preliminaries}\label{sec:prelim}

We begin with notation that will be used throughout the paper and provide some standard definitions from convex analysis and differential privacy.

\paragraph{Notation}
We let $n$ denote the sample size and $d$ the dimension. We let $\param$ denote the optimization variable and $\paramdomain \subset \R^d$ the constraint set. $\statval$ are samples from $\statdomain$, and $\statrv$ is an $\statdomain$-valued random variable. For each sample $\statval \in
\statdomain$, $F(\cdot; \statval): \R^d \rightarrow \R \cup \{+\infty\}$ is
a closed convex function. Let $\partial \risksamp(\param; \statval)$ denote the subdifferential of $\risksamp(\cdot; \statval)$ at $\param$. We let $\statdomain^n$ denote the collection of datasets $\statvalset= (\statval_1, \ldots, \statval_n)$ with $n$ data points from $\statdomain$. We let $\risk_\statvalset(\param) \defeq \frac{1}{n}\sum_{\statval \in \statvalset} \risksamp(\param, \statval)$ denote the empirical loss and $\risk(\param) \defeq \E[\risksamp(\param;\statrv)]$ denote the population loss. The distance of a point to a set is $\dist\paren{x,Y} = \min_{y \in Y}\ltwo{x - y}$. We use ${\rm Diam}(\paramdomain) = \sup_{x, y \in \paramdomain}\ltwo{x - y}$ to denote the diameter of parameter space $\paramdomain$ and use $D$ as a bound on the diameter of our parameter space. 

We recall the definition of $\ed$-differential privacy.
\begin{definition}
A randomized mechanism $M$ is $\ed$-differentially private ($\ed$-DP) if for all datasets $\statvalset, \statvalset' \in \statdomain^n$ that differ in a single data point and for all events $\mc{O}$ in the output space of $M$, we have 
\begin{align*}
    P(M(\statvalset)\in \mc{O}) \leq e^\diffp P(M(\statvalset') \in \mc{O}) + \delta.
\end{align*}
We define $\diffp$-differential privacy ($\diffp$-DP) to be $(\diffp, 0)$-differential privacy.
\end{definition}

We now recall a couple of standard convex analysis definitions.


\begin{definition}\label{def:lipschitz}
	~\\
	\begin{enumerate}
		\item A function $h: \xd \to \R$ is \emph{$\lip$-Lipschitz} if for all $x,y\in\paramdomain$
		\begin{align*}
			|{h(x) - h(y)}| \leq \lip \ltwo{x - y}.
		\end{align*}
		Equivalently, a function is \emph{$\lip$-Lipschitz} if $\ltwo{\nabla f(x)} \le \lip$ for all $x \in \xd$.
		\item A function $h$ is \emph{$\smooth$-smooth} if it has $\smooth$-Lipschitz  gradient: for all $x, y \in \paramdomain$
		\begin{align*}
			\ltwo{\nabla h(x) - \nabla h(y)} \leq \smooth \ltwo{x - y}.
		\end{align*}
	\item A function $h$ is \emph{$\growthcoef$-strongly convex} if for all $x, y \in \paramdomain$
	\begin{align*}
		h(y) \ge h(x) + \grad h(x)^T(y - x) + \frac{\growthcoef}{2}\ltwo{y - x}^2.
	\end{align*}
	\end{enumerate}
\end{definition}
%
%

We formally define interpolation problems:

\begin{definition}[Interpolation Problem]\label{def:interpolation}
Let $\paramdomain\opt \defeq \argmin_{\param \in \paramdomain} \risk(\param)$. Then problem \eqref{eqn:objective} is an interpolation problem if there exists $\param\opt \in \paramdomain\opt$ such that for $P$-almost all $\statval \in \statdomain$, we have $0 \in \partial F(x\opt;s)$.  
\end{definition}
Interpolation problems are common in modern machine learning, where models are  overparameterized. One simple example is overparameterized linear regression: there exists a solution that minimizes each individual sample function. Classification problems with margin are another example. 

Crucial to our results is the following quadratic growth assumption:
\begin{definition}\label{ass:growth}
We say that a function $f$ satisfies the quadratic growth condition if for all $x \in \xd$
\begin{align*}
    \risk(\param) - \inf_{\param' \in \paramdomain\opt}\risk(\param') \geq \frac{\growthcoef}{2}\dist\paren{\param,\paramdomain\opt}^2.
\end{align*}
\end{definition}
This assumption is natural with interpolation and holds for many important applications including noiseless linear regression~\cite{StrohmerVe09,NeedellWaSr14}. Past work (\cite{VaswaniBaSc19,WoodworthSr21}) uses this assumption with interpolation to get faster rates of convergence for non-private optimization.

Finally, the adaptivity of our algorithms will crucially depend on an innovation leveraging Lipchitizian extensions, defined as follows.
\begin{definition}[Lipschitzian extension \cite{HiriartUrrutyLe93ab}] \label{def:lip-ext}
The Lipschitzian extension with Lipschitz constant L of a function $f$ is defined as the infimal convolution
\begin{equation}
\label{eqn:lip-ext}
    f_L(x) \coloneqq \inf_{y \in \R^d} \{f(y) + L\ltwo{x - y}\}.
\end{equation}
\end{definition}

\noindent The Lipschitzian extension~\eqref{eqn:lip-ext} essentially transforms a general convex function into an $\lip$-Lipschitz convex function.
We now present a few properties of the Lipschitzian extension that are relevant to our development.

\begin{lemma}\label{lem:lip-ext}
Let $f:\paramdomain \to \R$ be convex. Then its Lipschitzian extension satisfies the following:
\begin{enumerate}
    \item $f_L$ is $L$-Lipschitz.
    \item $f_L$ is convex.
    \item If $f$ is $L$-Lipschitz, then $f_L(x) = f(x)$, for all $x$.
    \item Let $y(x) = \argmin_{y \in \R^d}\{f(y) + L\ltwo{x - y}\}$. If $y(x)$ is at a finite distance from $x$, we have
    \begin{equation*}
\grad f_L(x) = \begin{cases}
\grad f(x), &\text{if $\ltwo{\grad f(x)} \le L$}\\
L\frac{x - y(x)}{\ltwo{x - y(x)}}, &\text{otherwise}.
\end{cases}
\end{equation*}
\end{enumerate}
\end{lemma}

We use the Lipschitzian extension as a substitute for gradient clipping to ensure differential privacy. Unlike gradient clipping, which may alter the geometry of a convex problem to a non-convex one, the Lipschitzian extension of a function remains convex and thus retains other nice properties that we leverage in our algorithms in~\Cref{sec:growth}.




\section{Hardness of private interpolation}\label{sec:no-growth}
In non-private stochastic convex optimization, for smooth functions it is well known that interpolation problems enjoy the fast rate $O(1/n)$~\cite{SrebroSrTe10} compared to the minimax-optimal $O(1/\sqrt{n})$ without interpolation~\cite{Duchi18}. In this section, we show that such an improvement is not generally possible with privacy. The same lower bound of private non-interpolation problems, $d/n\diffp$, holds for interpolation problems.


To state our lower bounds, we present some notation that we will use throughout of the paper.
We let $\funcsetfamily$ denote the family of function $\risksamp$ and dataset $\statvalset$ pairs such that $\risksamp: \paramdomain \times \statdomain \rightarrow \R$ is convex and $\smooth$-smooth in its first argument, $|\statvalset| = n$, and $\risk_{\statvalset}(y) = \fracnsamp \sum_{\statval \in \statvalset}\risksamp(y, \statval)$ is an interpolation problem (\Cref{def:interpolation}). 
We define the constrained minimax risk to be

\iftoggle{arxiv}{
\begin{align*}
     \minimax(\paramdomain, \funcsetfamily, \varepsilon, \delta)
     \defeq 
     &\inf_{M \in \edfamily} \sup_{(\risksamp, \statvalset^n) \in \funcsetfamily} \E[\risk_{\statvalset^n}(M(\statvalset^n))] - \inf_{\param'\in\paramdomain}\risk_{\statvalset^n}(\param').
\end{align*}
}{
\begin{align*}
     &\minimax(\paramdomain, \funcsetfamily, \varepsilon, \delta)
     \defeq \\ 
     &\inf_{M \in \edfamily} \sup_{(\risksamp, \statvalset^n) \in \funcsetfamily} \E[\risk_{\statvalset^n}(M(\statvalset^n))] - \inf_{\param'\in\paramdomain}\risk_{\statvalset^n}(\param').
\end{align*}
}
where $\edfamily$ be the collection of $\ed$-differentially private mechanisms from $\statdomain^n$ to $\paramdomain$. We use $\efamily$ to denote the collection of $\diffp$-DP mechanisms from $\statdomain^n$ to $\paramdomain$. Here, the expectation is taken over the randomness of the mechanism, while the dataset $\statvalset^n$ is fixed. 

We have the following lower bound for private interpolation problems; the proof is deferred to \Cref{proof:thm:lb-private-interp-nogrowth}.
\begin{theorem}\label{thm:lb-private-interp-nogrowth}
Suppose $\paramdomain\subset \R^d$ contains a $d$-dimensional $\ell_2$ ball of diameter $\diam$. 
Then the following lower bound holds for $\delta=0$
\begin{align*}
     \minimaxe 
     \geq \frac{\smooth \diam^2 d}{96e^2 n\diffp }.
\end{align*}
Moreover, if $0<\delta < \diffp/6$ and $d=1$, the following lower bound holds
\begin{align*}
    \minimaxed
    \geq \frac{\smooth \diam^2 }{16(e + 1)n \diffp }.
\end{align*}
\end{theorem}
Recall the optimal rate for pure DP optimization problems without interpolation is $O(\frac{1}{\sqrt{n}} + \frac{d}{n\varepsilon})$. The first term is the non-private rate, as this is the rate one would get if $\varepsilon = 0$. The second term is the private rate, as this is the price algorithms have to pay for privacy. In modern machine learning, problems are often high dimensional, so we often think of the dimension $d$ scaling with some function of the number of samples $n$. Thus, the private rate is often thought to dominate the non-private rate. For this reason, in this section, we focus on the private rate. The lower bounds of~\Cref{thm:lb-private-interp-nogrowth} show that it is not possible to improve the private rate for interpolation problems in general. Similarly, for approximate \ed-DP, the lower bound shows that improvements are not possible for $d=1$. 
For completeness, as we alluded to earlier, we note that our results do not preclude the possibility of improving the non-private rate from $O(1/\sqrt{n})$ to $O(1/n)$. We leave this as an open problem of independent interest for future work.

Despite this pessimistic result, in the next section we show that substantial improvements are possible for private interpolation problems with additional growth conditions.

\section{Faster rates for interpolation with growth}\label{sec:growth}
Having established our hardness result for general interpolation problems, in this section we show that when the functions satisfy additional growth conditions, we get (nearly) exponential improvements in the rates of convergence for private interpolation. 

Our algorithms use recent localization techniques that yield optimal algorithms for DP-SCO~\cite{FeldmanKoTa20,AsiLeDu21} where the algorithm iteratively shrinks the diameter of the domain. However, to obtain faster rates for interpolation, we crucially build on the observation that the norm of the gradients is decreasing as we approach the optimal solution, since $\ltwo{\nabla F(x;s)} \le \smooth \ltwo{x-x\opt}$. Hence, by carefully localizing the domain and shrinking the Lipschitz constant accordingly, our algorithms improve the rates for interpolating datasets.

However, this technique alone yields an algorithm that may not be private for non-interpolation problems, violating that privacy must hold for all inputs: the reduction in the Lipschitz constant may not hold for non-interpolation problems, and thus, the amount of noise added may not be enough to ensure privacy. 
To solve this issue, we use the Lipschitzian extension (\Cref{def:lip-ext}) to transform  our potentially non-Lipschitz sample functions into Lipschitz ones and guarantee privacy even for non-interpolation problems. 

We begin in~\Cref{sec:lip-ext} by presenting our Lipschitzian extension based algorithm, which recovers the standard optimal rates for (non-interpolation) $L$-Lipschitz functions while still guaranteeing privacy when the function is not Lipschitz. Then in~\Cref{sec:growth-non-adap} we build on this algorithm to develop a localization-based algorithm that obtains faster rates for interpolation-with-growth problems. 
Finally, in~\Cref{sec:growth-adap} we present our final adaptive algorithm, which obtains fast rates for interpolation-with-growth problems while achieving optimal rates for non-interpolation growth problems.

\subsection{Lipschitzian-extension based algorithms }
\label{sec:lip-ext}
\newcommand{\dpalg}{\mathbf{M}^L_{(\diffp,\delta)}}
Existing algorithms for DP-SCO with $L$-Lipschitz functions may not be private if the input function is not  $\lip$-Lipschitz~\cite{BassilyFeGuTa20,FeldmanKoTa20,AsiLeDu21}. 
Given any DP-SCO algorithm $\dpalg$, which is private for $\lip$-Lipschitz functions, we present a framework that transforms $\dpalg$ to an algorithm which is (i) private for all functions, even ones which are not $\lip$-Lipschitz functions and (ii) has the same utility guarantees as $\dpalg$ for $\lip$-Lipschitz functions. In simpler terms, our algorithm essentially feeds $\dpalg$ the Lipschitzian-extension of the sample functions as inputs. \Cref{alg:lip-ext} describes our Lipschitzian-extension based framework. 



\begin{algorithm}
	\caption{Lipschitzian-Extension Algorithm}
	\label{alg:lip-ext}
	\begin{algorithmic}[1]
		\REQUIRE Dataset $\statvalset=(\ds_1, \ldots, \ds_n)\in \domain^n$;
		\STATE Let $F_L(x;s_i)$ be the Lipschitzian extension of $F(x;s_i)$ for all $i$.
		\begin{equation*}
		    F_L(x;s_i) = \inf_{y} \{F(y;s_i) + L\ltwo{x - y}\}
		\end{equation*}
		\STATE Run $\dpalg$ over the functions $F_L(\cdot;s_i)$.
		\STATE Let $x_{\rm priv}$ denote the output of $\dpalg$.
		\RETURN  $x_{\rm priv}$
	\end{algorithmic}
\end{algorithm} 

For this paper we consider $\dpalg$ to be Algorithm 2 of \cite{AsiLeDu21} (reproduced in \Cref{appen:asiledu-alg} as \Cref{alg:loc-growth}). The following proposition summarizes our guarantees for~\Cref{alg:lip-ext}.
\begin{proposition}
\label{prop:lip-ext-alg}
    Let $\mc{L}_L$ denote the set of sample function-dataset pair $(F,S)$ such that $F$ is $L$-Lipschitz and let $\mc{F}$ denote the set of sample function-dataset pair $(F,\statvalset)$ such that $\dpalg$ is $(\diffp,\delta)$-DP for  any $(F,\statvalset) \in \mc{L}_L \cap \mc{F}$. Then
    \begin{enumerate}
        \item For any $(F,\statvalset) \in \mc{F}$, \Cref{alg:lip-ext} is $(\diffp,\delta)$-DP.
        \item For any $(F,\statvalset) \in \mc{L}_L \cap \mc{F}$, \Cref{alg:lip-ext} achieves the same optimality guarantees as $\dpalg$.
    \end{enumerate}
\end{proposition}

\begin{proof}
    For the first item, note that~\Cref{lem:lip-ext} implies that $F_L$ is $L$-Lipschitz, i.e. $(F_L,\statvalset) \in \mc{L}_L \cap \mc{F}$. Since $\dpalg$ is $(\diffp,\delta)$-DP when applied over Lipschitz functions in $\mc{F}$, we have that \Cref{alg:lip-ext} is $(\diffp,\delta)$-DP. 
    
    For the second item, \Cref{lem:lip-ext} implies that $F_L = F$ when $F$ is $L$-Lipschitz. Thus, in \Cref{alg:lip-ext}, we apply $\dpalg$ over $F$ itself.
\end{proof}

While clipped DP-SGD does ensure privacy for input functions which are not $\lip$-Lipschitz,
our algorithm has some advantages over clipped DP-SGD: first, clipping does not result in optimal rates for pure DP, and second, clipped DP-SGD results in time complexity $O(n^{3/2})$. In contrast, our Lipschitzian extension approach is amenable to existing linear time algorithms~\cite{FeldmanKoTa20} allowing for almost linear time complexity algorithms for interpolation problems. 
Finally, while clipping the gradients and using the Lipschitzian extension both alter the effective function being optimized, only the Lipschitzian extension is able to preserve the convexity of said effective function (see item 2 in~\Cref{lem:lip-ext}).
We make a note about the computational efficiency of \Cref{alg:lip-ext}. Recall that when the objective is in fact $L$-Lipschitz, computing gradients for the Lipschitzian extension (say in the context of a first-order method) is only as expensive as computing the gradients for the original function. In particular, one can first compute the gradient of the original function and use item 4 of \Cref{lem:lip-ext}; when the problem is $L$-Lipschitz, $\|\nabla f(x)\|_2$ is always less than or equal to $L$ and thus the gradient of the Lipschitzian extension is just the gradient of the original function.

\subsection{Faster non-adaptive algorithm}
\label{sec:growth-non-adap}
Building on the Lipschitzian-extension framework of the previous section, in this section, we present our epoch based algorithm, which obtains faster rates in the interpolation-with-growth regime. It uses \Cref{alg:lip-ext} with $\dpalg$ as \Cref{alg:loc-growth} (reproduced in \Cref{appen:asiledu-alg}) as a subroutine
in each epoch, to localize and shrink the domain as the iterates get closer to the true minimizer. Simultaneously, the algorithm also reduces the Lipschitz constant, as the interpolation assumption implies that the norm of the gradient decreases for iterates near the minimizer. 
The detailed algorithm is given in \Cref{alg:priv-interpol-quad} where $\diam_i$ denotes the effective diameter and $\lip_i$ denotes the effective Lipschitz constant in epoch $i$. 

\begin{algorithm}[tb]
   \caption{Domain and Lipschitz Localization algorithm}
   \label{alg:priv-interpol-quad}
\begin{algorithmic}[1]
    \REQUIRE 
    Dataset $\statvalset=(\ds_1, \ldots, \ds_n)\in \domain^n$,
    Lipschitz constant $L$,
    domain $\paramdomain$,
    probability parameter $\tailprob$,
    initial point $\param_{0}$ 
    \STATE Set $L_1 = L$, $D_1 = {\rm Diam}(\paramdomain)$ and $\paramdomain_1 = \paramdomain$
    \STATE Partition the dataset into T partitions (denoted by $\{\statvalset_k\}_{k = 1}^T$) of size $m$ each; $\statvalset_k = (s_{(k-1)m + 1},\dots,s_{km})$
    \FOR{$i=1$ to $T$\,}
    \STATE $\param_i \leftarrow $  Run~\Cref{alg:lip-ext} with dataset $\statvalset_i$, constraint set $\paramdomain_i$,
    Lipschitz constant $L_i$, 
    probability parameter $\tailprob/T$, 
    privacy parameters $(\diffp,\delta)$, 
    initial point $\param_{i-1}$, 
    \STATE Shrink the diameter  
    
    \iftoggle{arxiv}
    {
     \begin{align*}
         D_{i+1} = 256 \left(\frac{L_i}{\growthcoef}\max\left\{\frac{\sqrt{\log(T/\tailprob)} \log^{3/2} \sampround}{\sqrt{\sampround}}\right.\right., \left.\left.\frac{\min(d,\sqrt{d \log(1/\delta)})\log(T/\tailprob) \log \sampround}{\sampround \diffp}\right\}\right)&
    \end{align*}
    }
    {
    \begin{align*}
         D_{i+1} = 256 \left(\frac{L_i}{\growthcoef}\max\left\{\frac{\sqrt{\log(T/\tailprob)} \log^{3/2} \sampround}{\sqrt{\sampround}}\right.\right.,&\\  \left.\left.\frac{\min(d,\sqrt{d \log(1/\delta)})\log(T/\tailprob) \log \sampround}{\sampround \diffp}\right\}\right)&
    \end{align*}
    }
    \STATE Set $\paramdomain_{i+1} = \{\param : \ltwo{\param - \param_i} \le \diam_{i+1}/2\}$
    \STATE Set $L_{i+1} = \smooth \diam_{i+1}$
	\ENDFOR
    \RETURN the final iterate $\param_T$
\end{algorithmic}
\end{algorithm}


The following theorem provides our upper bounds for~\Cref{alg:priv-interpol-quad}, demonstrating near-exponential rates for interpolation problems; we present the proof in \Cref{appen:proof-ub}.

\begin{restatable}{theorem}{ubquadtheorem}
\label{thm:ub-quad}
Assume each sample function $\risksamp$ is $L$-Lipschitz and $\smooth$-smooth, and let the population function $f$ satisfy quadratic growth (\Cref{ass:growth}). Let Problem \eqref{eqn:objective} be an interpolation problem. Then \Cref{alg:priv-interpol-quad} is $(\diffp,\delta)$-DP. 
For $\delta = 0$, 
 $\tailprob = \frac{1}{n^\mu}$, $\sampround = 256\log^2 n\frac{\smooth \log(1/\beta)}{\growthcoef}\max\left\{\frac{256\smooth }{\growthcoef},\frac{d}{\diffp \sqrt{\log n}}\right\}$, $T = n/\sampround$ and any $\mu > 0$,  \Cref{alg:priv-interpol-quad} returns $x_T$ such that 
\iftoggle{arxiv}{
\begin{align}
    \E[\risk(\param_T) - \risk(\param\opt)] \le   L\diam&\left(\frac{1}{n^\mu} + \exp\left(-\wt \Theta \paren{\frac{n \growthcoef^2}{\smooth^2}}\right) + \exp\left(- \wt \Theta \paren{\frac{\growthcoef n \diffp}{\smooth d}}\right)\right).\label{eqn:pure-ub}
\end{align}
}
{
\begin{align}
    \nonumber \E[\risk(\param_T) - \risk(\param\opt)] \le   L\diam&\left(\frac{1}{n^\mu} + \exp\left(-\wt \Theta \paren{\frac{n \growthcoef^2}{\mu\smooth^2}}\right) + \right. \\
    & \left.\exp\left(- \wt \Theta \paren{\frac{\growthcoef n \diffp}{\mu\smooth d}}\right)\right).\label{eqn:pure-ub}
\end{align}s
}
  
 For $\delta > 0$, 
  $\tailprob = \frac{1}{n^\mu}$, $\sampround = 256\log^2 n \frac{\smooth \log(1/\beta)}{\growthcoef}\max\left\{\frac{256\smooth }{\growthcoef},\frac{\sqrt{d}\log(1/\delta)}{\diffp \sqrt{\log n}}\right\}$, $T = n/\sampround$ and any $\mu > 0$, \Cref{alg:priv-interpol-quad} returns $x_T$ such that 
 \iftoggle{arxiv}{
 \begin{align}
    \E[\risk(\param_T) - \risk(\param\opt)] \le L\diam\left(\frac{1}{n^\mu} + \exp\left(\wt \Theta \paren{\frac{n \growthcoef^2}{\smooth^2}}\right) + \exp\left(- \wt \Theta \paren{\frac{\growthcoef n \diffp}{\smooth \sqrt{d \log(1/\delta)}}}\right) 
    \right).\label{eqn:appr-ub}
\end{align}
}
{
 \begin{align}
    \nonumber \E[\risk(\param_T) - \risk(\param\opt)] &\le L\diam\left(\frac{1}{n^\mu} + \exp\left(\wt \Theta \paren{\frac{n \growthcoef^2}{\smooth^2}}\right) + \right.\\
    &\left.\exp\left(- \wt \Theta \paren{\frac{\growthcoef n \diffp}{\smooth \sqrt{d \log(1/\delta)}}}\right) 
    \right).\label{eqn:appr-ub}
\end{align}}
\end{restatable}
The exponential rates in~\Cref{thm:ub-quad} show a significant improvement in the interpolation regime over the minimax-optimal $O( ( \frac{1}{\sqrt{n}} + \frac{d}{n \diffp} )^2)$ without interpolation~\cite{FeldmanKoTa20,AsiLeDu21}.
To get the linear convergence rates, we run roughly $ n/\log n$ epochs with $\log n$ samples each. Thus, each call of the subroutine runs the algorithm on only logarithmic number of samples compared to the number of epochs.
Intuitively, growth conditions improves the performance of the sub-algorithm, while growth and interpolation conditions reduce the search space. This in tandem leads to faster rates.

To better illustrate the improvement in rates compared to the non-private setting, the next corollary states the private sample complexity required to achieve error $\alpha$ in the interpolation regime.

\begin{corollary}\label{cor:ub-quad}
Let the conditions of \Cref{thm:ub-quad} hold. For $\delta = 0$ , \Cref{alg:priv-interpol-quad} is $\diffp$-DP and requires 
\begin{align*}
    n = \wt O \paren{\frac{1}{\alpha^{\rho}} + \frac{d}{\rho \diffp}\log\paren{\frac{1}{\alpha}}}
\end{align*}
samples to ensure $\E[\risk(\param_T) - \risk(\param\opt)] \le \alpha$
for any fixed $\rho > 0$, where $\wt O$ ignores only polyloglog factors in $1/\alpha$. \\
Moreover, for $\delta > 0$, \Cref{alg:priv-interpol-quad} is \ed-DP and requires  
\begin{align*}
    n = \wt O \paren{\frac{1}{\alpha^{\rho}} + \frac{\sqrt{d\log(1/\delta)}}{\rho\diffp}\log\paren{\frac{1}{\alpha}}}
\end{align*}
samples to ensure $\E[\risk(\param_T) - \risk(\param\opt)] \le \alpha$, for any fixed $\rho > 0$, where $\wt O$ ignores polyloglo factors in $1/\alpha$.
\end{corollary}
As the sample complexity of DP-SCO to achieve expected error $\alpha$ on general quadratic growth problems is~\cite{AsiLeDu21}
\begin{equation*}
    \Theta \left( \frac{1}{\alpha} + \frac{d}{\diffp \sqrt{\alpha}} \right),
\end{equation*}
\Cref{cor:ub-quad} shows that we are able to improve the polynomial dependence on $1/\alpha$ in the sample complexity to (nearly) logarithmic for interpolation problems.

\begin{rem}
In contrast to \Cref{cor:ub-quad}, we can tune the failure probability parameter $\beta$ to get the sample complexity $\frac{d}{\diffp}\log^2\paren{\frac{1}{\alpha}}$. Even though this sample complexity does not have the polynomial factor, it may be worse than $\frac{1}{\alpha^{\rho}} + \frac{d}{\diffp}\log\paren{\frac{1}{\alpha}}$, because generally the dimension term is the dominant one.
\end{rem}

We end this section by considering growth conditions that are weaker than quadratic growth. 
\begin{rem}(interpolation with $\kappa$-growth)
We can extend our algorithms to work for 
the weaker $\growth$-growth condition \cite{AsiLeDu21}, i.e., $f(x) - f(x^\star) \ge  \frac{\lambda}{\growth} \norms{x-x^\star}_2^\growth$. We present the full details of these algorithms in \Cref{appen:gen-growth} (see \Cref{alg:priv-interpol-kappa}). In this setting, we obtain excess loss
\begin{equation*}
O\left( \left( \frac{1}{\sqrt{n}} + \frac{d}{n \diffp} \right)^{\frac{\kappa}{\kappa - 2}} \right), 
\end{equation*}
for interpolation problems, improving over the minimax-optimal loss for non-interpolation problems which is
\begin{equation*}
O\left( \left( \frac{1}{\sqrt{n}} + \frac{d}{n \diffp} \right)^{\frac{\kappa}{\kappa - 1}} \right). 
\end{equation*}
As an example, when $\kappa=3$, this corresponds to an improvement from roughly $(d/n\diffp)^{3/2}$ to $(d/n\diffp)^{3}$. Like our previous results, we are again able to show similar improvements for \ed-DP with better dependence on the dimension. Finally, we note that we have not provided lower bounds for the interpolation-with-$\kappa$-growth setting for $\kappa>2$. We leave this question as a direction for future research. 


\end{rem}

\subsection{Adaptive algorithm}
\label{sec:growth-adap}

Though \Cref{alg:priv-interpol-quad} is private and enjoys faster rates of convergence in the interpolation regime, it is not necessarily adaptive to interpolation, i.e.~it may perform poorly given a non-interpolation problem. In fact, since the shrinkage of the diameter and Lipschitz constants at each iteration hinges squarely on the interpolation assumption, 
the new domain may not include the optimizing set $\paramdomain\opt$ in the non-interpolation setting, so our algorithm may not even converge. Since in general we do not know a priori whether a dataset is interpolating, it is important to have an algorithm which adapts to interpolation.

To that end, we present an adaptive algorithm that achieves faster rates for interpolation-with-growth problems while simultaneously obtaining the standard optimal rates for general growth problems. 
The algorithm consists of two steps. In the first step, our algorithm privately minimizes the objective without assuming it is an interpolation problem. Next, we run our non-adaptive interpolation algorithm of~\Cref{sec:growth-non-adap} over the localized domain returned by the first step. If our problem was an interpolating one, the second step recovers the faster rates in \Cref{sec:growth-non-adap}. If our problem was not an interpolating one, 
the first localization step ensures that we at least recover the non-interpolating convergence rate. We stress that the privacy of \Cref{alg:priv-adapt-interpol} requires that the call to \Cref{alg:priv-interpol-quad} remains private even if the problem is non-interpolating. This is ensured by using our Lipschitzian extension based algorithm with $\dpalg$ as \Cref{alg:loc-growth}. The Lipschitzian extension allows us to continue preserving privacy. We present the full details of this algorithm in~\Cref{alg:priv-adapt-interpol}.


\begin{algorithm}[tb]
   \caption{Algorithm that adapts to interpolation}
   \label{alg:priv-adapt-interpol}
\begin{algorithmic}[1]
    \REQUIRE 
    Dataset $\statvalset=(\ds_1, \ldots, \ds_n)\in \domain^n$,
    Lipschitz constant $L$,
    domain $\paramdomain$,
    probability parameter $\tailprob$,
    initial point $\param_{0}$
    \STATE Partition the dataset into 2 partitions $S_1 = (s_{1},\dots,s_{n/2})$ and $S_2 = (s_{(n/2)+1},\dots,s_{n})$
    \STATE $\param_1 \leftarrow $  Run \Cref{alg:lip-ext} with dataset $S_1$,
    constraint set $\paramdomain_i$,
    Lipschitz constant $L_i$, 
    probability parameter $\tailprob/2$, 
    privacy parameters $(\diffp,\delta)$, 
    initial point $\param_{i-1}$, 
    \STATE Shrink the diameter
     \iftoggle{arxiv}{\begin{align*}
         \diam_{\rm int} = \frac{128 L}{\lambda} \cdot &\left( \frac{ \sqrt{\log(2/\beta) } \log^{3/2} n}{\sqrt{n}} + \frac{\min\{d,\sqrt{d\log(1/\delta)}\}\log(2/\beta)\log n}{n \diffp} \right)
     \end{align*}
     }
     {\begin{align*}
         \diam_{\rm int} = \frac{128 L}{\lambda} \cdot &\left( \frac{ \sqrt{\log(2/\beta) } \log^{3/2} n}{\sqrt{n}}\right.+\\
          &\left.\frac{\min\{d,\sqrt{d\log(1/\delta)}\}\log(2/\beta)\log n}{n \diffp} \right)
     \end{align*}}
     
    \STATE $\paramdomain_{\rm int} = \{\param : \ltwo{\param - \param_1} \le \diam_{\rm int}/2\}$
    \STATE $\param_{\rm adapt} \leftarrow $ Run \Cref{alg:priv-interpol-quad} with dataset $S_2$,
    diameter $\diam_{\rm int}$,
    Lipschitz constant $L$,
    domain $\paramdomain_{\rm int}$,
    smoothness parameter $\smooth$,
    tail probability parameter $\tailprob/2$,
    growth parameter $\growthcoef$,
    initial point $\param_{1}$
    \RETURN the final iterate $\param_{\rm adapt}$.
\end{algorithmic}
\end{algorithm}

The following theorem (\Cref{thm:adapt-conv}) states the convergence guarantees of our adaptive algorithm (\Cref{alg:priv-adapt-interpol}) in both the interpolation and non-interpolation regimes for the pure DP setting. The results for approximate DP are similar and can be obtained by replacing $d$ with $\sqrt{d\log(1/\delta)}$; we give the full details in \Cref{appen:proof-ub}. 


\begin{restatable}{theorem}{ubadapttheorem}
\label{thm:adapt-conv}
Let each sample function $\risksamp$ be $L$-Lipschitz and $\smooth$-smooth, and let the population function $f$ satisfy quadratic growth (\Cref{ass:growth}) with coefficient $\growthcoef$. Let $\param_{\rm adapt}$ be the output of \Cref{alg:priv-adapt-interpol}. Then
\begin{enumerate}
    \item \Cref{alg:priv-adapt-interpol} is $\diffp$-DP. 
    \item Without any additional interpolation assumption, $\param_{\rm adapt}$ satisfies
    \begin{align*}
     \E[\risk(\param_T) - \risk(\param\opt)] \le   L\diam \cdot \wt O \left( \frac{ 1}{\sqrt{n}} + \frac{d}{n \diffp} \right)^2.
    \end{align*}
    \item Let problem \eqref{eqn:objective} be an interpolation problem. Then$\param_{\rm adapt}$ satisfies
    \iftoggle{arxiv}{
    \begin{align*}
    \nonumber \E[\risk(\param_T) - \risk(\param\opt)] \le   L\diam&\left(\frac{1}{n^\mu} + \exp\left(- \wt \Theta \paren{\frac{n \growthcoef^2}{\smooth^2}}\right)  + \exp\left(- \wt \Theta \paren{\frac{\growthcoef n \diffp}{\smooth d}}\right)\right).
\end{align*}}
{
\begin{align*}
    \nonumber \E[\risk(\param_T) - \risk(\param\opt)] \le   L\diam&\left(\frac{1}{n^\mu} + \exp\left(- \wt \Theta \paren{\frac{n \growthcoef^2}{\smooth^2}}\right)  \right. \\
    & + \left.\exp\left(- \wt \Theta \paren{\frac{\growthcoef n \diffp}{\smooth d}}\right)\right).
\end{align*}
}
\end{enumerate}
\end{restatable}
\begin{proof}
The privacy of \Cref{alg:priv-adapt-interpol} follows from the privacy of \Cref{alg:lip-ext,alg:priv-interpol-quad} and post-processing.

To prove the convergence guarantees, we first need to show that the optimal set $\paramdomain\opt$ is in the shrinked domain $\paramdomain_{\rm int}$. Using the high probability guarantees of \Cref{alg:lip-ext}, we know that with probability $1 - \beta/2$, we have
\iftoggle{arxiv}{
\begin{align*}
    f(x_1) - f(x\opt) \le \frac{2^{12} L}{\lambda} \cdot &\left( \frac{ \sqrt{\log(2/\beta) } \log^{3/2} n}{\sqrt{n}} + \frac{\sqrt{d\log(1/\delta)}\log(2/\beta)\log n}{n \diffp} \right).
\end{align*}}
{
\begin{align*}
    f(x_1) - f(x\opt) \\ \le \frac{2^{12} L}{\lambda} \cdot &\left( \frac{ \sqrt{\log(2/\beta) } \log^{3/2} n}{\sqrt{n}}\right.+\\
          &\left.\frac{\sqrt{d\log(1/\delta)}\log(2/\beta)\log n}{n \diffp} \right)
\end{align*}
}
Using the quadratic growth condition, we immediately have $\ltwo{\param\opt - \param_1} \le \diam_{\rm int}/2$ and hence $\paramdomain\opt \subset \paramdomain_{\rm int}$.

Using smoothness, we have that for any $x \in \paramdomain_{\rm int}$,
\begin{align*}
    f(x) - f(x\opt) \le \frac{H\diam_{\rm int}^2}{2}.
\end{align*}
Since \Cref{alg:priv-interpol-quad} always outputs a point in its input domain (in this case $\paramdomain_{\rm int}$), even in the non-interpolation setting that
\begin{align*}
    \E[\risk(\param_T) - \risk(\param\opt)] \le   L\diam \cdot \wt O \left( \frac{ 1}{\sqrt{n}} + \frac{d}{n \diffp} \right)^2.
\end{align*}

\noindent In the interpolation setting, the guarantees of \Cref{alg:priv-interpol-quad} hold and result is immediate.
\end{proof}

\section{Optimality and Superefficiency}\label{sec:super}

We conclude this paper by providing a lower bound and a super-efficiency result that demonstrate the tightness of our upper bounds. Recall that our upper bound from~\Cref{sec:growth} is roughly (up to constants)
\begin{equation}
\label{eq:ub}
\frac{1}{n^c} + \exp\left(-\wt \Theta \left(\frac{n \diffp}{d} \right) \right),
\end{equation}
for any arbitrarily large $c$. We begin with an exponential lower bound showing that the second term in~\eqref{eq:ub} is tight. 
We then prove a superefficiency result that demonstrates that any  private algorithm which avoids the first term in~\eqref{eq:ub} cannot be adaptive to interpolation, that is, it can not achieve the minimax optimal rate for the family of non-interpolation problems. 


\Cref{thm:lb-private-interp-growth} below presents our exponential lower bounds for private interpolation problems with growth.
We use the notation and proof structure of~\Cref{thm:lb-private-interp-nogrowth}. We let $\funcsetfamilygrowth\subset\funcsetfamily$ be the subcollection of function, data set pairs which also have functions $\risk_{\statvalset^n}$ that have $\lambda$-quadratic growth (\Cref{ass:growth}).
The proof of \Cref{thm:lb-private-interp-growth} is found in \Cref{proof:thm:lb-private-interp-growth}.

\begin{theorem}\label{thm:lb-private-interp-growth}
Let $\paramdomain\subset \R^d$ contain a $d$-dimensional $\ell_2$-ball of diameter $\diam$. 
Then
\begin{align*}
    \minimaxegrowth
    \geq \frac{\lambda \diam^2}{96}\exp\left(-\frac{2\lambda n\varepsilon}{\smooth d}\right).
\end{align*}
\end{theorem}

\newcommand{\superfamily}{\funcsetfamily_\lambda^L(\risksamp)}
This lower bound addresses the second term of~\eqref{eq:ub}; we now turn to our superefficiency results to lower bound the first term of~\eqref{eq:ub}. We start with defining some notation and making some simplifying assumptions.
For a fixed function $\risksamp: \paramdomain, \statdomain \rightarrow \R$ which is convex, $\smooth$-smooth with respect to the first argument, let $\superfamily$ be the set of datasets $\statvalset$ of $n$ data points sampled from $\statdomain$ such that $\risk_{\statvalset}(\param) \defeq \frac{1}{n}\sum_{\statval \in \statvalset^n} \risksamp(\param, \statval)$ is $\lip$-Lipschitz and have $\lambda$-strongly convex objectives. 
For simplicity, we will assume that 1. $\inf_{\param\in\paramdomain}\risksamp(\param; \statval) = 0$ for all $\statval \in \statdomain$, 2. $\paramdomain = [-\diam, \diam] \subset \R$, and 3. the codomain of $F$ is $\R_+$.
With this setup, we present the formal statement of our result; the proof of \Cref{thm:superefficiency} is found in \Cref{proof:thm:superefficiency}.  

\begin{theorem}\label{thm:superefficiency}
Suppose we have some $\statvalset \in \superfamily$ with $\lipschitz = 2\smooth\diam$ such that $(\risksamp, \statvalset)$ satisfy \Cref{def:interpolation}.
Suppose there is an $\diffp$-DP estimator $M$ such that
\begin{align*}
    \E[\risk_\statvalset(M(\statvalset)) ]- \inf_{\param \in \xd}\risk_\statvalset(\param) \leq c\diam^2 
    e^{-\Theta((n\diffp)^t)}
\end{align*}
for some $t > 0 $ and absolute constant $c$.
Then, for sufficiently large $n$, there exists another dataset $\statvalset' \in \superfamily$, where $(\risksamp, \statvalset')$ may \textbf{not} satisfy \Cref{def:interpolation}, such that  
\begin{align*}
    \E[\risk_{\statvalset'}(M(\statvalset'))] - \inf_{\param \in \xd}\risk_{\statvalset'}(\param) = 
    \Omega\left(\frac{\diam^2}{(n\diffp)^{2(1-t)}}\right) 
\end{align*}
\end{theorem}

To better contextualize this result, suppose there exists an algorithm which atttains a $\exp(-\wt \Theta \left(n \diffp/d \right) )$ convergence rate on interpolation problems; i.e., the algorithm is able to avoid the $1/n^c$ term in~\eqref{eq:ub}. Then~\Cref{thm:superefficiency} states that there exists some strongly convex, non-interpolation problem on which the aforementioned algorithm will optimize very poorly; in particular, the algorithm will only be able to return a solution that attains, on average, constant error on this ``hard'' problem. More generally, recall that in the non-interpolation quadratic growth setting, the optimal error rate is on the order of $1/(n\diffp)^2$~\cite{AsiLeDu21}. \Cref{thm:superefficiency} shows that attaining better-than-polynomial error complexity on quadratic growth interpolation problems implies that the algorithm cannot be minimax optimal in the non-interpolation quadratic growth setting. Thus, the rates our adaptive algorithms attain are the best we can hope for if we want an algorithm to perform well on both interpolation and non-interpolation quadratic growth problems. 


\bibliography{bib,interpol-bib}
\bibliographystyle{alpha}

\appendix

\section{Results from previous work}

\subsection{Proof of \Cref{lem:lip-ext}}
\begin{enumerate}
    \item Follows from Proposition IV.3.1.4 of \cite{HiriartUrrutyLe93ab}.
    \item Follows from Proposition IV.3.1.4 of \cite{HiriartUrrutyLe93ab}.
    \item Follows since for $L$-lipschitz functions $ 0 \in \grad f(x) + L\mathbb{B}_2$.
    \item Follows from Section VI.4.5 of \cite{HiriartUrrutyLe93ab}.
\end{enumerate}

\subsection{Algorithms from \cite{AsiLeDu21}}
\label{appen:asiledu-alg}
\begin{algorithm}
	\caption{Localization based Algorithm}
	\label{alg:pure-erm}
	\begin{algorithmic}[1]
		\REQUIRE Dataset $D=(\ds_1, \ldots, \ds_n)\in \domain^n$,
		constraint set $\xdomain$,
		step size $\ss$, initial point $x_0$, 
		Lipschitz (clipping) constant $L$,
		privacy parameters $(\diffp,\delta)$;
		\STATE Set $k = \ceil{\log n}$ and $n_0 = n/k$
		\FOR{$i=1$ to $k$\,}
		\STATE Set  
		$\ss_i = 2^{-4i} \ss$ 
		\STATE Solve the following ERM over 
		$ \mc{X}_i= \{x\in \xdomain: \norm{x - x_{i-1}}_2 \le {2\lip \ss_i n_0} \}$:
		\begin{equation*}
		F_i(x) =  \frac{1}{n_0} \sum_{j=1 + (i-1)n_0}^{in_0} \f(x;\ds_j) + \frac{1}{\ss_i n_0 } \norm{x - x_{i-1}}_2^2
		\end{equation*}
		\STATE Let $\hat x_i$ be the output of the optimization algorithm.
		\IF{$\delta=0$}
		\STATE Set 
		$\noise_i \sim \laplace_d(\sigma_i)$ where $\sigma_i = 4 \lip \ss_i \sqrt{d}/\diffp_i$
		\ELSIF{$\delta>0$}
		\STATE Set $\noise_i \sim \normal(0,\sigma_i^2)$ where $\sigma_i = 4 \lip \ss_i \sqrt{\log(1/\delta)}/\diffp$
		\ENDIF
		\STATE Set $x_i = \hat x_i + \noise_i$
		\ENDFOR
		\RETURN the final iterate $x_k$
	\end{algorithmic}
\end{algorithm} 

\begin{algorithm}
	\caption{Epoch-based algorithms for $\kappa$-growth}
	\label{alg:loc-growth}
	\begin{algorithmic}[1]
		\REQUIRE 
		Dataset $\Ds=(\ds_1, \ldots, \ds_n)\in \domain^n$,
		constraint set $\xdomain$,
		Lipschitz (clipping) constant $L$,
		initial point $x_0$,
		number of iterations $T$,
		probability parameter $\tailprob$,
		privacy parameters $(\diffp,\delta)$;
		\STATE Set $n_0 = n/T$ and $\rad_0 = {\rm diam}(\xdomain)$
		\IF{$\delta=0$}
		\STATE Set 
		$\ss_0 = \frac{\rad_0}{2\lip} \min \left(\frac{1}{\sqrt{n_0 \log(n_0) \log(1/\beta)}} ,\frac{ \diffp}{ d \log(1/\beta)} \right)$
		\ELSIF{$\delta>0$}
		\STATE Set 
		\begin{align*}
		    \ss_0 = \frac{\rad_0}{2\lip} \min \left\{\frac{1}{\sqrt{n_0 \log(n_0) \log(1/\beta)}} , \frac{ \diffp}{ \sqrt{d \log(1/\delta)} \log(1/\beta)} \right)
		\end{align*}
		\ENDIF
		\FOR{$i=0$ to $T - 1$\,}
		\STATE Let $\Ds_i = (\ds_{1+(i-1) n_0}, \dots, \ds_{i n_0})$
		\STATE Set $\rad_i = 2^{-i} \rad_0$ and $\ss_i = 2^{-i} \ss_0$
		\STATE Set $\xdomain_i = \{x \in \xdomain : \ltwo{x - x_i} \le \rad_i \}$
		\STATE Run~\Cref{alg:pure-erm} on dataset $\Ds_i$ with starting point $x_i$, Lipschitz (clipping) constant $L$,
		 privacy parameter $(\diffp,\delta)$, domain $\xdomain_i$ (with diameter $\rad_i$), step size $\ss_i$ 
		\STATE Let $x_{i+1}$ be the output of the private procedure 
		\ENDFOR
		\RETURN $x_T$
	\end{algorithmic}
\end{algorithm} 


\subsection{Theoretical results from \cite{AsiLeDu21}}

We first reproduce the high probability guarantees of \Cref{alg:pure-erm} as proved in \cite{AsiLeDu21}. 

\begin{restatable}{proposition}{restatePureSCOHb}
 	\label{thm:sco-hb-pure}
 	Let $\beta \le 1/(n+d)$, $\diam_2(\xdomain)\le \rad$ and $\f(x;\ds)$ be convex, $\lip$-Lipschitz for all $\ds \in \domain$. Setting
 	\begin{equation*}
 	\ss = \frac{\rad}{\lip} \min \left(\frac{1}{\sqrt{n \log(1/\beta)}} ,\frac{ \diffp}{ d \log(1/\beta)} \right)
 	\end{equation*}
 	then for $\delta=0$, \Cref{alg:pure-erm} is $\diffp$-DP and has with probability $1-\beta$
 	\begin{equation*}
 	f(x) -  f(x\opt) 
 	\le 128\lip \rad \cdot  \left( \frac{\sqrt{\log(1/\beta) } \log^{3/2} n}{\sqrt{n}} + \frac{ d \log(1/\beta) \log n}{n \diffp} \right).
 	\end{equation*}
\end{restatable} 

\begin{restatable}{proposition}{restateApproxSCOHb}
 	\label{thm:sco-hb-appr}
 	Let $\beta \le 1/(n+d)$, $\diam_2(\xdomain)\le \rad$ and $\f(x;\ds)$ be convex, $\lip$-Lipschitz for all $\ds \in \domain$. Setting
 	\begin{equation*}
 	\ss = \frac{\rad}{\lip} \min \left(\frac{1}{\sqrt{n \log(1/\beta)}} ,\frac{ \diffp}{ \sqrt{d \log(1/\delta)} \log(1/\beta)} \right),
 	\end{equation*}
 	then for $\delta > 0 $, \Cref{alg:pure-erm} is $(\diffp,\delta)$-DP and has with probability $1-\beta$
 	\begin{equation*}
 	f(x) -  f(x\opt) 
 	\le 128 \lip \rad \cdot  \left( \frac{\sqrt{\log(1/\beta) } \log^{3/2} n}{\sqrt{n}} + \frac{ \sqrt{d \log(1/\delta)} \log(1/\beta) \log n}{n \diffp} \right).
	\end{equation*}
\end{restatable}

Now, we reproduce the high probability convergence guarantees of \Cref{alg:loc-growth}.

\begin{restatable}{theorem}{restateSCOGrowthPure}
	\label{thm:sco-growth-pure}
	Let $\beta \le 1/(n+d)$, $\diam_2(\paramdomain)\le \diam$ and $\f(x;\ds)$ be convex, $\lip$-Lipschitz for all $\ds \in \statdomain$. 
	Assume that $\pf$ has $\kappa$-growth (Assumption~\ref{ass:growth}) with $\kappa \ge \lkappa >1$.
	Setting $T = \left\lceil \frac{2 \log n}{\lkappa-1} \right\rceil$,	
	\Cref{alg:loc-growth} is $\diffp$-DP and has with probability $1-\beta$
	\begin{equation*}
	\pf(x_T) -  \min_{x \in \xdomain} \pf(x) 
	\le  \frac{4032}{\lambda^{\frac{1}{\kappa - 1}}} \cdot  \left( \frac{\lip \sqrt{\log(1/\beta) }\log^{3/2} n }{\sqrt{n}} + \frac{ \lip d \log(1/\beta)\log n}{n \diffp (\lkappa -1)} \right)^{\frac{\kappa}{\kappa-1}}.
	\end{equation*}
\end{restatable}

\begin{restatable}{theorem}{restateSCOGrowthApprox}
	\label{thm:sco-growth-appr}
	Let $\beta \le 1/(n+d)$, $\diam_2(\xdomain)\le \rad$ and $\f(x;\ds)$ be convex, $\lip$-Lipschitz for all $\ds \in \statdomain$. 
	Assume that $\pf$ has $\kappa$-growth (Assumption~\ref{ass:growth}) with $\kappa \ge \lkappa >1$.
	Setting $T = \left\lceil \frac{2 \log n}{\lkappa-1} \right\rceil$ and $\delta>0$, \Cref{alg:loc-growth} is $(\diffp,\delta)$-DP and has with probability $1-\beta$
	\begin{equation*}
	\pf(x_T) -  \min_{x \in \xdomain} \pf(x)
	\le \frac{4032}{\lambda^{\frac{1}{\kappa - 1}}} \cdot \left( \frac{\lip \sqrt{\log(1/\beta) } \log^{3/2} n}{\sqrt{n}} + \frac{ \lip \sqrt{d \log(1/\delta)} \log(1/\beta)\log n}{n \diffp (\lkappa -1)} \right)^{\frac{\kappa}{\kappa-1}}.
	\end{equation*}
\end{restatable} 



\section{Proofs from \Cref{sec:no-growth}}

\label{appen:proof-lb}

\subsection{Proof of \Cref{thm:lb-private-interp-nogrowth}}\label{proof:thm:lb-private-interp-nogrowth}
Consider the sample risk function 
\begin{align*}
    \risksamp(\param; \statval) \defeq \frac{\smooth}{2} \ltwo{\param - \statval}^2 \cdot \ind{s \neq 0}.
\end{align*}
We define the datasets $\statvalset_v^n \defeq \{0\}^{n - k}\cup \{v\}^{k}$. We define the corresponding population risk to be $\risk_v(\param) \defeq \frac{1}{n} \sum_{\statval \in \statvalset_v} \risksamp(\param; \statval) = \frac{k\smooth}{2n} \ltwo{\param - v}^2 $. We select $\mc{V}$ to be a $\gamma$-packing (with respect to the $\ell_2$ norm) of diameter $\diam$ ball contained in $\paramdomain$. Define the separation between $v, v' \in \mc{V}$ with respect to the loss $\risk_v$ and $\risk_{v'}$ by
\begin{align*}
d_{\rm opt}(f_v,f_{v'}) \coloneqq \inf_{\param \in \paramdomain} \frac{\risk_v(\param)}{2} + \frac{\risk_{v'}(\param)}{2} \geq c \defeq \frac{k \smooth}{8n}\gamma^2.
\end{align*}

For the sake of contradiction, suppose that $\E[\risk_v(M(\statvalset_v))] \leq \tau$ for $\tau <  \frac{k \smooth \gamma^2}{8n(1 + e^{k\varepsilon}2^d\gamma^d/\diam^d)}$ for all $v \in \paramdomain$. Then by Markov's inequality, $\P(f_v(M(\statvalset_{v})) > c) \leq \frac{\tau}{c}$ and  $\P(f_v(M(\statvalset_{v})) \leq c) \geq 1 - \frac{\tau}{c}$ for all $v$, and so
\begin{align*}
    \frac{\tau}{c} 
        & \stackrel{(i)} {\geq} \P(\risk_v(M(\statvalset_{v})) > c) \\
        & \stackrel{(ii)}{\geq} \P(\cup_{v' \in \mc{V} \setminus \{v\}}\risk_{v'}(M(\statvalset_{v})) \leq c)  \\
        & \stackrel{(iii)}{\geq} e^{-k\varepsilon}\sum_{v' \in \mc{V} \setminus \{v\}} \P(\risk_{v'}(M(\statvalset_{v'}))  \leq c)  \\ 
        & \stackrel{(i)}{\geq} e^{-k\varepsilon} (|\mc{V}| - 1)\left(1 - \frac{\tau}{c}\right),
\end{align*}
where inequality $(ii)$ follows from the definition of the separation, and $(iii)$ follows from privacy and the disjoint nature of the events in the union. Rearranging, we get that 
\begin{align*}
    \tau \geq \frac{k \smooth \gamma^2}{8n(1 + e^{k\varepsilon}(|\mc{V}| -1)\inv)},
\end{align*}
which is a contradiction.  By standard packing inequalities~\cite{Wainwright19}, we know that $|\mc{V}| \geq (\diam/2\gamma)^d$. Setting $k = d/\varepsilon$ and $\gamma = \diam/2e$ and using the fact that $x/ (x-1)$ is decreasing in $x$ gives 
\begin{align*}
    \tau \geq \frac{ d \smooth \diam^2}{32n\diffp e^2(1 + e^{d}(e^d -1)\inv)}  \geq \frac{\smooth \diam^2 d}{96e^2 n\diffp }.
\end{align*}

We now prove the $\ed$-DP lower bound. Consider the following sample risk function 
\begin{align*}
    \risksamp(\param; \statval) \defeq \frac{\smooth}{2} (\param - \statval)^2  \ind{s \neq 0}.
\end{align*}
We define the datasets $\statvalset_v^n \defeq \{0\}^{n - k}\cup \{v\}^{k}$ inducing the corresponding population risk $\risk_v(\param) \defeq \frac{1}{n} \sum_{\statval \in \statvalset_v} \risksamp(\param; \statval) = \frac{k\smooth}{2n} (\param - v)^2 $. We select two points $v, v'$ contained within the diameter $D$ ball contained in $\paramdomain$ such that $|v - v'| = D$. Define the separation between $v, v' \in \mc{V}$ with respect to the loss $\risk_v$ and $\risk_{v'}$ as
\begin{align*}
d_{\rm opt}(f_v,f_{v'}) \coloneqq \inf_{\param \in \paramdomain} \frac{\risk_v(\param)}{2} + \frac{\risk_{v'}(\param)}{2} \geq c \defeq \frac{k \smooth}{8n}\diam^2 
\end{align*}

For the sake of contradiction, suppose that $\E[\risk_v(M(\statvalset_v))] \leq \tau$ for $\tau < \frac{k \smooth \diam^2}{8n}\left\lparen\frac{e^{-k\diffp} - k e^{-\diffp} \delta}{1 + e^{-k\varepsilon}}\right\rparen$ for all $v \in \paramdomain$.  Then by Markov's inequality, $\P(f_v(M(\statvalset_{v})) > c) \leq \frac{\tau}{c}$ and  $\P(f_v(M(\statvalset_{v})) \leq c) \geq 1 - \frac{\tau}{c}$ for all $v$, and so
\begin{align*}
    \frac{\tau}{c} 
        & \stackrel{(i)} {\geq} \P(\risk_v(M(\statvalset_{v})) > c) \\
        & \stackrel{(ii)}{\geq} \P(\risk_{v'}(M(\statvalset_{v})) \leq c)  \\
        & \stackrel{(iii)}{\geq} e^{-k\varepsilon} \P(\risk_{v'}(M(\statvalset_{v'}))  \leq c) - ke^{-\diffp} \delta \\ 
        & \stackrel{(i)}{\geq} e^{-k\varepsilon} \left(1 - \frac{\tau}{c}\right) - ke^{-\diffp} \delta,
\end{align*}
where inequality $(ii)$ follows from the definition of the separation, and $(iii)$ follows from group privacy of $\ed$-privacy \cite{DworkRo14}. Rearranging, we get that 
\begin{align*}
    \tau \geq  \frac{k \smooth \diam^2}{8n}\left\lparen\frac{e^{-k\diffp} - k e^{-\diffp} \delta}{1 + e^{-k\varepsilon}}\right\rparen,
\end{align*}
which is a contradiction. Setting $k = 1/\diffp$ and using the fact $\delta \leq \diffp e^{\diffp - 1} / 2$ gives the first result.

\section{Proofs from \Cref{sec:growth}}

\label{appen:proof-ub}

We first prove a lemma that each time we shrink the domain size, the set of interpolating solutions still lies in the new domain with high probability, and the new Lipschitz constant we define is a valid Lipschitz constant for the loss defined on the new domain. We prove it in generality for $\growth$-growth.

\begin{lemma}\label{lem:valid-subalg}
Let $\paramdomain\opt$ denote the set of  interpolating solutions of problem \eqref{eqn:objective}. Then $\paramdomain\opt \subset \paramdomain_{i}$ for all $i \in [T]$ with probability $1 - \beta$, and $\ltwo{\grad F(y;s)} \le L_i$ for all $y \in \paramdomain_i$.
\end{lemma}
\begin{proof}
 We prove this lemma for the case when $\delta = 0$, the case when $\delta > 0$ follows similarly. For epoch $i$, using Theorem 2 of \cite{AsiLeDu21}, we have with probability $1 - \tailprob/T$,
 \begin{align*}
    \risk(\hparam_i) - \risk(\param\opt) \le \frac{C_\kappa}{\growthcoef^{\frac{1}{\growth - 1
    }}}\max\left\{\frac{L_i\sqrt{\log(T/\tailprob)}\log^{3/2} \sampround}{ \sqrt{\sampround}},\frac{L_i{d}\log(T/\tailprob)\log \sampround}{\sampround \diffp}\right\}^{\frac{\growth}{\growth - 1}}
\end{align*}

Using the growth condition on $f(\cdot)$, we have
\begin{align*}
    \ltwo{\hparam_i - \param\opt} \le \sqrt[\growth]{\frac{\growth(\risk(\hparam_i) - \risk(\param\opt))}{\growthcoef}}  \le (C_\growth \growth)^{1/\growth}\max\left\{\frac{L_i\sqrt{\log(T/\tailprob)}\log^{3/2} \sampround}{\growthcoef \sqrt{\sampround}},\frac{L_i{d}\log(T/\tailprob)\log \sampround}{\growthcoef \sampround \diffp}\right\}^{\frac{1}{\growth - 1}},
\end{align*}
Using $c_\growth = 2 (C_\growth \growth)^{1/\growth}$, we get $\ltwo{\hparam_i - \param\opt} \le D_{i+1}/2$ with probability $1 - \tailprob/T$. Thus, for each epoch $i$, with probability $1 - \tailprob/T$, each point in the set $\paramdomain\opt$ of optimizers lies in the domain $\paramdomain_{i}$. Using a union bound on all epochs, we have $\paramdomain\opt \subset \paramdomain_{i}$ for all $i \in [T]$ with probability $1 - \tailprob$. 

We now prove the second part of the lemma. Using the smoothness of $F(\cdot;s)$ and that $\grad F(\param\opt;s) = 0$ for all $\param\opt \in \paramdomain\opt$, we have
\begin{align*}
    \ltwo{\grad\risksamp(y;\sampval)} = \ltwo{\grad\risksamp(y;\sampval) - \grad\risksamp(\param\opt;\sampval)} \le \smooth \ltwo{y - \param\opt} \leq \smooth \left(\ltwo{y - \hparam_i} + \ltwo{\hparam\opt - \hparam_i}\right) \le \smooth D_{i} = L_i
\end{align*}
as desired.
\end{proof}

We now restate and prove the convergence rate of \Cref{alg:priv-interpol-quad}

\ubquadtheorem*
\begin{proof}
First we prove the privacy guarantee of the algorithm. Each sample impacts only one of the iterates $\hparam_i$, thus Algorithm \ref{alg:priv-interpol-quad} satisfies the same privacy guarantee as \Cref{alg:loc-growth} by postprocessing.

We divide the utility proof into 2 main parts; first is to check the validity of the assumptions while applying \Cref{alg:loc-growth} and second is using its high probability convergence guarantees to get the final rates. To check this, we ensure that the optimum set lies in the new domain $\paramdomain_i$ at step $i$ and that the Lipschitz constant $L_i$ defined with respect to the domain is a valid lipschitz constant. This follows from \Cref{lem:valid-subalg}.

Next, we use the high probability convergence guarantees of the subalgorithm \Cref{alg:loc-growth} to get convergence rates for \Cref{alg:priv-interpol-quad}.
We prove it for the case when $\delta = 0$, the case when $\delta > 0$ is similar. We know that
\begin{align*}
 L_i & = \smooth D_i \\
     & = c_2\frac{\smooth L_{i-1}}{\growthcoef} \max\left\{\frac{\sqrt{\log(T/\tailprob)} \log^{3/2} \sampround}{\sqrt{\sampround}},\frac{d\log(T/\tailprob) \log \sampround}{\sampround \diffp}\right\}. 
\end{align*}
 Thus we have
 \begin{align*}
     L_T = \left(c_2\frac{\smooth}{\growthcoef} \max\left\{\frac{\sqrt{\log(T/\tailprob)} \log^{3/2} \sampround}{\sqrt{\sampround}},\frac{d\log(T/\tailprob) \log \sampround}{\sampround \diffp}\right\}\right)^{T-1} L_1.
 \end{align*}
Using Theorem 2 of \cite{AsiLeDu21} on the last epoch, we have with probability $1 - \tailprob$ that
\begin{align*}
     \risk(\hparam_T) - \risk(\param\opt) &\le C_2\frac{L^2_T}{\growthcoef}\max\left\{\frac{\log(T/\tailprob)\log^{3/2}\sampround}{ \sampround},\frac{{d^2}\log^2(T/\tailprob) \log\sampround}{ \sampround^2 \diffp^2}\right\} \\
     & =\left(c^2_2\frac{\smooth^2}{\growthcoef^2} \max\left\{\frac{\log(T/\tailprob)\log^{3/2}\sampround}{ \sampround},\frac{{d^2}\log^2(T/\tailprob) \log\sampround}{ \sampround^2 \diffp^2}\right\}\right)^{T} \frac{C_2 L_1^2\growthcoef}{\smooth^2c_2^2}\\
     & =\left(c^2_2\frac{\smooth^2}{\growthcoef^2} \max\left\{\frac{\log(T/\tailprob)\log^{3/2}\sampround}{ \sampround},\frac{{d^2}\log^2(T/\tailprob) \log\sampround}{ \sampround^2 \diffp^2}\right\}\right)^{T} \frac{ L_1^2\growthcoef}{8\smooth^2}.
\end{align*}

Let $\sampround = k \log^2 n$ and $T = n/\sampround$ for some $k$ such that 
$$\left(c^2_2\frac{\smooth^2}{\growthcoef^2} \max\left\{\frac{\log(n/(\tailprob k\log^2 n))\log^{3/2}(k \log^2 n)}{k \log^2 n}, \frac{{d^2}\log^2(n/(\tailprob k\log^2 n)) \log(k \log^2 n)}{ (k \log^2 n)^2 \diffp^2}\right\}\right) \le \frac{1}{e} .$$
This holds for example for 
\begin{align*}
    k = 256\frac{\smooth \log(1/\beta)}{\growthcoef}\max\left\{\frac{256\smooth }{\growthcoef},\frac{d}{\diffp \sqrt{\log n}}\right\},
\end{align*}
for sufficiently large $n$. Using these values of $\sampround$ and $T$, we have
\begin{align}\label{eqn:good-conv}
     \risk(\hparam_T) - \risk(\param\opt) \le \frac{C_2 L^2\growthcoef}{\smooth^2c_2^2} \exp\left(-\frac{n}{k \log^2 n}\right) = \frac{L_1^2\growthcoef}{8\smooth^2} \exp\left(-\frac{n}{k \log^2 n}\right).
\end{align}

To get the convergence results in expectation, let $A$ denote the ``bad" event with tail probability $\tailprob$, where $ \risk(\hparam_T) - \risk(\param\opt) > \frac{L_1^2\growthcoef}{8\smooth^2} \exp\left(-\frac{n}{k \log^2 n}\right)$. Now,
\begin{align*}
    \E[\risk(\hparam_T) - \risk(\param\opt)] &\le \beta\frac{\smooth\diam^2}{2} + (1 - \beta)\E[\risk(\hparam_T) - \risk(\param\opt) \mid A^c]\\
    &\le \beta\frac{\smooth\diam^2}{2} + \E[\risk(\hparam_T) - \risk(\param\opt) \mid A^c]
\end{align*}
Substituting $\tailprob = \frac{1}{n^\mu}$ and using \Cref{eqn:good-conv}, we get the result.
\end{proof}





\subsection{Algorithm for general $\growth$}

\label{appen:gen-growth}

\begin{algorithm}[ht]
   \caption{Epoch based epoch based epoch based clipped-GD}
   \label{alg:priv-interpol-kappa}
\begin{algorithmic}[1]
    \REQUIRE 
    number of epochs: $\epochs$, samples in each round: $\sampround = n/T$, Diameter at the start: $\diam_1$, lipschitz constant at the start $L_1$, domain $\paramdomain_1$, initial point $\hparam_{0}$
    \FOR{$i=1$ to $T$\,}
    \STATE $\hparam_i \leftarrow $  Output of \Cref{alg:loc-growth} when run on domain $\paramdomain_i$ (diameter $\diam_i$), with lipschitz constant $L_i$ using $\sampround$ samples.
    \IF{$\delta = 0$}
    \STATE  \begin{align*}
        \text{Set} \ D_{i+1} = c_\growth \left(\frac{L_i}{\growthcoef}\max\left\{\frac{\sqrt{\log(T/\tailprob)} \log^{3/2} \sampround}{\sqrt{\sampround}},\frac{d\log(T/\tailprob) \log \sampround}{\sampround \diffp}\right\}\right)^{\frac{1}{\growth - 1}}
    \end{align*}
    \ELSIF{$\delta > 0$}
    \STATE \begin{align*}
        \text{Set} \ D_{i+1} = c_\growth \left(\frac{L_i}{\growthcoef}\max\left\{\frac{\sqrt{\log(T/\tailprob)}\log^{3/2} \sampround}{\sqrt{\sampround}},\frac{\sqrt{d \log(1/\delta)}\log(T/\tailprob)\log \sampround}{\sampround \diffp}\right\}\right)^{\frac{1}{\growth - 1}}
    \end{align*}
    \ENDIF
    \STATE Set $\paramdomain_{i+1} = \{\hparam : \ltwo{\hparam - \hparam_i} \le \diam_{i+1}/2\}$
    \STATE Set $L_{i+1} = \smooth \diam_{i+1}$
	\ENDFOR
    \RETURN the final iterate $\param_T$
\end{algorithmic}
\end{algorithm}

\begin{remark}
$c_\growth$ is an absolute constant dependent on the high probability performance guarantees of \Cref{alg:loc-growth}. We can calculate that $C_\growth$ is at most $2^{12}(\sim 4000)$ and hence $c_\growth \le 2(2^{12}\growth)^{1/\growth} \le 4 \cdot 2^{12/\growth}$.
\end{remark}

\begin{restatable}{theorem}{ubkappatheorem}
\label{thm:ub-kappa}
Assume each sample function $\risksamp$ be $L$-Lipschitz and $\smooth$-smooth, and let the population function $f$ satisfy quadratic growth (\Cref{ass:growth}). Let Problem \eqref{eqn:objective} be an interpolation problem. Then, \Cref{alg:priv-interpol-kappa} is $(\diffp,\delta)$-DP. 
For $\delta = 0$, 
\Cref{alg:priv-interpol-kappa} with $T = \log n$ and $m = \frac{n}{\log n}$, we have
 \begin{align*}
    \risk(\hparam_T) - \risk(\param\opt) \le  \widetilde{O}\paren{\frac{1}{\sqrt{n}} + \frac{d}{n\diffp}}^{\frac{\growth}{\growth - 2}},
 \end{align*}
 with probability $1 - \tailprob$. For $\delta > 0$, \Cref{alg:priv-interpol-kappa} when run using $T = \log n$ and $\sampround = n/\log n$ achieves error
 \begin{align*}
    \risk(\hparam_T) - \risk(\param\opt) \le  \widetilde{O}\paren{\frac{1}{\sqrt{n}} + \frac{\sqrt{d}\log(1/\delta)}{n\diffp}}^{\frac{\growth}{\growth - 2}},
 \end{align*}
  with probability $1 - \tailprob$.
\end{restatable}

\begin{proof}
The privacy guarantee follows from the proof of \Cref{thm:ub-quad}. We divide the utility proof into 2 main parts; first is to check the validity of the assumptions while applying \Cref{alg:loc-growth} and second is using its high probability convergence guarantees to get the final rates. To check this, we ensure that the optimum set lies in the new domain defined at every step and that the lipschitz constant defined with respect to the domain is a valid lipschitz constant. This follows from \Cref{lem:valid-subalg}.

Next, we use the high probability convergence guarantees of the subalgorithm \Cref{alg:loc-growth} to get convergence rates for \Cref{alg:priv-interpol-quad}.

We prove it for the case when $\delta = 0$, the case when $\delta > 0$ is similar. We know that
\begin{align*}
    L_i & = \smooth D_i \\
     & = c_\growth\smooth\left(\frac{ L_{i-1}}{\growthcoef} \max\left\{\frac{\sqrt{\log(T/\tailprob)} \log^{3/2} \sampround}{\sqrt{\sampround}},\frac{d\log(T/\tailprob) \log \sampround}{\sampround \diffp}\right\}\right)^{\frac{1}{\growth - 1}}.
\end{align*}
 Thus, we have
\begin{align*}
    L_T & = (c_\growth\smooth)^{\frac{\growth - 1}{\growth - 2}\left(1 - \frac{1}{(\growth - 1)^{T-1}}\right)}\left(\frac{1}{\growthcoef} \max\left\{\frac{\sqrt{\log(T/\tailprob)} \log^{3/2} \sampround}{\sqrt{\sampround}},\frac{d\log(T/\tailprob) \log \sampround}{\sampround \diffp}\right\}\right)^{\frac{1}{\growth - 2}\left(1 - \frac{1}{(\growth - 1)^{T-1}}\right)} L_1^\frac{1}{(\growth - 1)^{T-1}}.
\end{align*}
 We note that for $T\sim \log n$, $\frac{1}{(\growth - 1)^{T-1}} \approx \frac{1}{n^{\log(\kappa - 1)}}$ and thus for large $n$, we ignore the terms of the form $a^{-\frac{1}{n^{\log(\kappa - 1)}}}$ since they are $\approx 1$. Ignoring these terms by including an additional constant $C'$ we can write
 \begin{align*}
    L_T & = C'(c_\growth\smooth)^{\frac{\growth - 1}{\growth - 2}}\left(\frac{1}{\growthcoef} \max\left\{\frac{\sqrt{\log(T/\tailprob)} \log^{3/2} \sampround}{\sqrt{\sampround}},\frac{d\log(T/\tailprob) \log \sampround}{\sampround \diffp}\right\}\right)^{\frac{1}{\growth - 2}} L_1^\frac{1}{(\growth - 1)^{T-1}}.
\end{align*}
 
 Using Theorem 2 of \cite{AsiLeDu21} on the last epoch, we have with probability $1 - \tailprob$ that
\begin{align*}
     \risk(\hparam_T) - \risk(\param\opt) &\le \frac{C_\kappa }{\growthcoef^{\frac{1}{\growth - 1
    }}}\max\left\{\frac{L_T\sqrt{\log(T/\tailprob)}\log^{3/2} \sampround}{ \sqrt{\sampround}},\frac{L_T{d}\log(T/\tailprob)\log \sampround}{\sampround \diffp}\right\}^{\frac{\growth}{\growth - 1}} \\
     & = \frac{(C')^{\frac{\growth}{\growth - 1}}C_\kappa (c_\growth\smooth)^{\frac{\growth}{\growth - 2}}}{\growthcoef^{\frac{2}{\growth - 2
    }}}\max\left\{\frac{\sqrt{\log(T/\tailprob)}\log^{3/2} \sampround}{ \sqrt{\sampround}},\frac{{d}\log(T/\tailprob)\log \sampround}{\sampround \diffp}\right\}^{\frac{\growth}{\growth - 2}} L_1^{{\frac{\growth}{(\growth - 1)^T}}}.
\end{align*}
Choosing $T = \log n$ and $\sampround = n/\log n$, we have
\begin{align*}
     \risk(\hparam_T) - \risk(\param\opt) &\le \frac{(C')^{\frac{\growth}{\growth - 1}}C_\kappa (c_\growth\smooth)^{\frac{\growth}{\growth - 2}}}{\growthcoef^{\frac{2}{\growth - 2
    }}}\max\left\{\frac{\sqrt{\log(\log n/\tailprob)}\log^{3/2} (n/\log n)}{ \sqrt{n/\log n}},\frac{{d}\log(\log n/\tailprob)\log (n/\log n)}{\diffp n/\log n }\right\}^{\frac{\growth}{\growth - 2}} L_1^{{\frac{\growth}{n}}}.
\end{align*}

Now we write results in terms of sample complexity required to achieve a particular error. The sufficient number of samples. To ensure $\risk(\hparam_T) - \risk(\param\opt) < \alpha$, it is sufficient to ensure 
$$ \frac{(C')^{\frac{\growth}{\growth - 1}}C_\kappa (c_\growth\smooth)^{\frac{\growth}{\growth - 2}}}{\growthcoef^{\frac{2}{\growth - 2
    }}}\max\left\{\frac{\sqrt{\log(T/\tailprob)}\log^{3/2} \sampround}{ \sqrt{\sampround}},\frac{{d}\log(T/\tailprob)\log \sampround}{\sampround \diffp}\right\}^{\frac{\growth}{\growth - 2}} L_1^{{\frac{\growth}{(\growth - 1)^T}}} < \alpha .$$

Choosing $n = \tilde{O}\left(\max\{(\frac{1}{\alpha^2})^\frac{\growth - 2}{\growth},(\frac{d}{\diffp\alpha})^\frac{\growth - 2}{\growth}\}\right)$ ensures error $\le \alpha$.
\end{proof}

\begin{restatable}{corollary}{ubkappaexp}
\label{cor:ub-kappa}
Under the conditions of \Cref{thm:ub-kappa}, for $\delta = 0$, the expected error of the output of algorithm is upper bounded by
\begin{align*}
    \E[\risk(\hparam_T) - \risk(\param\opt)] \le \widetilde{O}\paren{\frac{1}{\sqrt{n}} + \frac{d}{n\diffp}}^{\frac{\growth}{\growth - 2}},
\end{align*}
for arbitrarily large $\mu$. For $\delta > 0$, the expected error of the output of algorithm is upper bounded by
\begin{align*}
    \E[\risk(\hparam_T) - \risk(\param\opt)] \le \widetilde{O}\paren{\frac{1}{\sqrt{n}} + \frac{d}{n\diffp}}^{\frac{\growth}{\growth - 2}},
\end{align*}
for arbitrarily large $\mu$.
\end{restatable}

\subsection{$(\diffp,\delta)$ version of \Cref{thm:adapt-conv}}
\begin{restatable}{theorem}{ubadapttheoremed}
\label{thm:adapt-conv-appr}
Assume each sample function $\risksamp$ be $L$-Lipschitz and $\smooth$-smooth, and let the population function $f$ satisfy quadratic growth (\Cref{ass:growth}) with coefficient $\growthcoef$. Let $\param_{\rm adapt}$ be the output of \Cref{alg:priv-adapt-interpol}. Then, 
\begin{enumerate}
    \item \Cref{alg:priv-adapt-interpol} is $\diffp$-DP. 
    \item Without any additional interpolation assumption, we have that the expected error of the $\param_{\rm adapt}$ is upper bounded by 
    \begin{align*}
     \E[\risk(\param_T) - \risk(\param\opt)] \le   L\diam \cdot \wt O \left( \frac{ 1}{\sqrt{n}} + \frac{\sqrt{d\log(1/\delta)}}{n \diffp} \right)^2.
    \end{align*}
    \item Let problem \eqref{eqn:objective} be an interpolation problem. Then, the expected error of the $\param_{\rm adapt}$ is upper bounded by 
    \begin{align*}
    \nonumber \E[\risk(\param_T) - \risk(\param\opt)] \le   L\diam&\left(\frac{1}{n^\mu} + \exp\left(- \wt \Theta \paren{\frac{n \growthcoef^2}{\smooth^2}}\right)  \right. \\
    & + \left.\exp\left(- \wt \Theta \paren{\frac{\growthcoef n \diffp}{\smooth \sqrt{d\log(1/\delta)}}}\right)\right).
\end{align*}
\end{enumerate}
\end{restatable}
\begin{proof}
First, we note that the privacy of \Cref{alg:priv-adapt-interpol} follows from the privacy of \Cref{alg:priv-interpol-quad} and \Cref{alg:lip-ext} and post-processing.

To prove the convergence guarantees, we first need to show that the optimal set $\paramdomain\opt$ is included in the shrinked domain $\paramdomain_{\rm int}$. Using the high probability guarantees of \Cref{alg:lip-ext}, we know that with probability $1 - \beta/2$, we have
\begin{align*}
    f(x_1) - f(x\opt)  \le \frac{2^{12} L}{\lambda} \cdot ^{\log(\kappa - 1)}\left( \frac{ \sqrt{\log(2/\beta) } \log^{3/2} n}{\sqrt{n}}\right.+
          \left.\frac{\sqrt{d\log(1/\delta)}\log(2/\beta)\log n}{n \diffp} \right).
\end{align*}
Using the quadratic growth condition, we immediately have $\ltwo{\param\opt - \param_1} \le \diam_{\rm int}/2$ and hence $\paramdomain\opt \subset \paramdomain_{\rm int}$.

Using smoothness, we have that for any $x \in \paramdomain_{\rm int}$,
\begin{align*}
    f(x) - f(x\opt) \le \frac{H\diam_{\rm int}^2}{2}.
\end{align*}
Since \Cref{alg:priv-interpol-quad} always outputs a point in its input domain (in this case $\paramdomain_{\rm int}$), even in the non-interpolation setting we have that
\begin{align*}
    \E[\risk(\param_T) - \risk(\param\opt)] \le   L\diam \cdot \wt O \left( \frac{ 1}{\sqrt{n}} + \frac{\sqrt{d\log(1/\delta)}}{n \diffp} \right)^2.
\end{align*}

In the interpolation setting, the guarantees of \Cref{alg:priv-interpol-quad} hold and the result is immediate.
\end{proof}

\section{Proofs from \Cref{sec:super}}
\subsection{Proof of \Cref{thm:lb-private-interp-growth}}\label{proof:thm:lb-private-interp-growth}
The proof is exactly the same as \Cref{thm:lb-private-interp-nogrowth}, except we set $k = \frac{\lambda n}{\smooth}$ to ensure that $\risk_v(\param)$ for any $v \in \paramdomain$ has $\lambda$-quadratic growth. Finally we set $\gamma = \frac{\diam}{2}\exp(\frac{-\lambda n\varepsilon}{\smooth d})$ and use the fact that $e^{\frac{\lambda n\diffp}{\smooth}}\geq 2$ and the fact that $x/ (x-1)$ is decreasing in $x$ to give the desired lower bound.

\subsection{Proof of \Cref{thm:superefficiency}}\label{proof:thm:superefficiency}

The proof of this result hinges on the two following supporting propositions. We first copy Proposition 2.2 from \cite{AsiDu20} (listed as \Cref{proposition:super-efficiency} below) in our notation for convenience. We then state \Cref{thm:modulus-of-continuity} which gives upper and lower bounds on the modulus of continuity (defined in \Cref{proposition:super-efficiency}). We note that the lower bound presented in \Cref{thm:modulus-of-continuity} is one of the novel contributions of this paper; the proof of \Cref{thm:modulus-of-continuity} can be found in \Cref{proof:thm:modulus-of-continuity}. We will first assume this to be true and prove \Cref{thm:superefficiency} before returning prove its correctness. 

\begin{proposition}
  \label{proposition:super-efficiency}
  For some fixed $\risksamp: \paramdomain, \statdomain \rightarrow \R$ which is convex and $\smooth$-smooth with respect to its first argument,
  let
  $\statvalset \in \superfamily$ for $\lipschitz = 2\smooth\diam$. Let $\param_\statvalset\opt = \argmin_{\param' \in \paramdomain} \risk_\statvalset(\param')$. Define the corresponding modulus of continuity
  \begin{align*}
    \lmod(\statvalset, 1/\varepsilon) \defeq \sup_{\statvalset'\in \superfamily} \{|\param_\statvalset\opt - \param_{\statvalset'}\opt| : \dham(\statvalset, \statvalset') \leq 1/\varepsilon \}.
\end{align*} 
  Assume the mechanism $M$ is $\diffp$-DP and
  for some $\gamma \le \frac{1}{2e}$ achieves
  \begin{equation*}
    \E[|M(\statvalset) - \param_\statvalset\opt|]
    \le \gamma \left(\frac{\lmod(\statvalset; 1 / \diffp)}{2}\right).
  \end{equation*}
  Then there exists a sample $\statvalset' \in \superfamily$ where
  $\dham(\statvalset, \statvalset') \le \frac{\log(1/2\gamma)}{2 \diffp}$ such that
  \begin{align*}
    \E[|M(\statvalset') -  \param_{\statvalset'}\opt|] 
    \ge \frac{1}{4} \ell\left(\frac{1}{4} \lmod\left(\statvalset';
    \frac{\log(1 / 2\gamma)}{2 \diffp}\right)\right).
  \end{align*}
\end{proposition}

\begin{proposition}\label{thm:modulus-of-continuity}
  Let $\risksamp: \paramdomain, \statdomain \rightarrow \R$ be convex and $\smooth$-smooth in its first argument and satisfying $\inf_{\param\in\paramdomain}\risksamp(\param; \statval) = 0$ for all $\statval \in \statdomain$. Suppose we have some $\statvalset \in \superfamily$ with $\lipschitz = 2\smooth\diam$ which also induces an interpolation problem (a problem which satisfies \Cref{def:interpolation}).
With respect to the dataset $\statvalset$, the modulus of continuity $\lmod(\statvalset, 1/\varepsilon)$
satsifies
\begin{align*}
    \frac{\diam}{ n \varepsilon} \leq \lmod(\statvalset, 1/\varepsilon)  \leq  \frac{8\smooth \diam}{\lambda n \varepsilon}
\end{align*}
\end{proposition}

With these two results, we can now prove \Cref{thm:superefficiency}. Restating the conditions of the theorem formally, suppose for some constants $c_0$ and $c_1$ there is an $\diffp$-DP estimator $M$ such that
\begin{align*}
    \E[\risk_\statvalset(M(\statvalset)) ]- \inf_{\param \in \xd}\risk_\statvalset(\param) \leq c_0\diam^2 
    e^{-c_1(n\diffp)^t}.
\end{align*}
If $t > 1$, set $t = \min(1, t)$, then the bound certainly still holds for large enough $n$. 
If we let $\param_\statvalset\opt = \argmin_{\param\in\paramdomain} \risk_\statvalset(\param)$, using the definition of strong convexity, we have that there exists some $c_2$ and $c_3$ such that 
\begin{align*}
    \E[|M(\statvalset) - \param_\statvalset\opt|] \leq c_2\diam e^{- c_3(n\diffp)^t}
\end{align*}
To satisfy the expression from \Cref{proposition:super-efficiency}, we select $\gamma$ such that
\begin{align*}
    \frac{\gamma\lmod(\statvalset; 1/\varepsilon)}{2} =  c_2 \diam e^{-c_3 (n\diffp)^t}.
\end{align*}
Using \Cref{thm:modulus-of-continuity}
we must have $\frac{ \lambda n \varepsilon}{4 \smooth} c_2  \exp(-c_3 (n\diffp)^t) \leq \gamma \leq 2  n\varepsilon  c_2  \exp(-c_3 (n\diffp)^t)$. Using \Cref{proposition:super-efficiency}, we have that 
\begin{align*}
    \E[|M(\statvalset') - \param_{\statvalset'}\opt|] \geq \lmod\left(\statvalset' ; \frac{\log(1/2\gamma)}{2\varepsilon}\right)
\end{align*}
Before performing a further lower bound on this quantity, we first verify that $\frac{\log(1/2\gamma)}{2\varepsilon}$ does not exceed the total size of the dataset, $n$. Using our bounds on $\gamma$, we see that
\begin{align*}
    \frac{\log(1/2\gamma)}{2\varepsilon} \leq \frac{1}{2\varepsilon}\left( c_3 (n\diffp)^t  - \log c_2 - \log\left(\frac{\lambda n \varepsilon}{2 \smooth}\right)\right)
\end{align*}
For any $t\in(0,1]$, for sufficiently large $n$, this quantity is less than $n$. We now lower bound the modulus of continuity by using the fact that it is a non-decreasing function in its second argument:
\begin{align*}
    \E[|M(\statvalset') - \param\opt_{\statvalset'}|] &\geq \lmod\left(\statvalset' ; \frac{\log(1/2\gamma)}{2\diffp}\right) \geq \lmod\left(\statvalset' ; \frac{ c_3 (n\diffp)^t  - \log c_2 - \log(4n\diffp)}{2\diffp}\right)\\
    &\geq \frac{\diam}{2n\diffp}\left[c_3(n\diffp)^t - \log c_2 - \log(4n\diffp) \right].
\end{align*}
This is the desired result; the last inequality comes from another application of \Cref{thm:modulus-of-continuity} but with $\frac{ c_3 (n\diffp)^t  - \log c_2 - \log(4n\varepsilon)}{2\diffp}$ in place of $1/\diffp$.

\subsubsection{Proof of \Cref{thm:modulus-of-continuity}}\label{proof:thm:modulus-of-continuity}

\noindent\fbox{%
    \parbox{\textwidth}{%
    \textbf{Proof Outline}\\~\\
    At a high level, starting with a function $\riskn$, we first remove an arbitrary $1/\varepsilon$ fraction to create a function $\riskneps$. We then replace the sample functions we removed with $1/\varepsilon$ samples of $\frac{\smooth}{2}(\param-\diam)^2$ and argue how far the minimizer of $\riskneps +\frac{\smooth}{2n\varepsilon}(\param-\diam)^2$ is away from the minimizer of $\riskn$. We will need many supporting lemmas to complete this proof; we quickly outline how we use these lemmas.
\begin{enumerate}
    \item We use \Cref{lem:growth-minimizer-stability} to argue that the minimizers of $\riskneps$ are no different that $\riskn$.
    \item We use \Cref{lem:sc-closure-removal} to argue about the growth of $\riskneps$.
    \item We use \Cref{lem:sc-closure-add}, \Cref{lem:growth-smooth-consequence}, \Cref{lem:control-of-opt-2} to lower bound how far the minimizer of $\riskneps +\frac{\smooth}{2n\varepsilon}(\param-\diam)^2$ has moved from the minimizer of $\riskn$. 
    \item We use \Cref{lem:stability} to upper bound how far the minimizer of $\riskneps +\frac{\smooth}{2n\varepsilon}(\param-\diam)^2$ has moved from the minimizer of $\riskn$. 
\end{enumerate}
}}\\~\\
We now formally introduce the several supporting lemmas which will aid our proof of \Cref{thm:modulus-of-continuity}. The first ensures that the minimizing set does not change upon the removal of a constant number of samples.

\begin{lemma}\label{lem:growth-minimizer-stability}
Assume that $\inf_{\param\in\paramdomain}\risksamp(\param; \statval) = 0$ for all $\statval \in \statdomain$.
Suppose $\riskn$ satisfies \Cref{def:interpolation} and has $\lambda$-quadratic growth. Let $\paramdomain\opt\defeq \argmin_{\param\in\paramdomain} \riskn(\param)$. Let $\statvalset_\varepsilon \subset \statvalset$ consist of any (constant not scaling with $n$) $1/\varepsilon > 0$ data points. Then, for $\riskneps \defeq \fracnsamp \sum_{\statval \in \statvalset \setminus \statvalset_\varepsilon} \risksamp(\param; \statval)$ we have that $\paramdomainneps \defeq \argmin_{\param\in\paramdomain} \riskneps(\param) = \paramdomain\opt$.
\end{lemma}
\begin{proof}
Suppose for the sake of contradiction that $\paramdomain\opt \neq \paramdomainneps$
Since $\riskn$ is an interpolation problem, the removal of samples can only increase the size of $\paramdomainneps$. Suppose that $\paramdomainneps \setminus \paramdomain\opt \neq \emptyset$. There exists at most $1/\varepsilon$ points in $\statvalset$ that have non-zero error on $\paramdomainneps \setminus \paramdomain\opt$. However, by smoothness of each sample function (and the fact that $\risk(\param\opt) = 0$ and $\risk'(\param\opt) = 0$ by construction), we have that for $\param \in [a,b]$
\begin{align*}
    \riskn(\param) \leq \frac{\smooth }{n\varepsilon}\dist(\param, \paramdomain\opt)^2.
\end{align*}
Since $\lim_{n\to\infty}\frac{\smooth}{n\varepsilon} = 0$, this contradicts $\lambda$-quadratic growth.
\end{proof}

This second lemma ensures that deleting a constant number of samples does not affect the growth or strong convexity of the population function by too much.

\begin{lemma}\label{lem:sc-closure-removal}
Assume that $\inf_{\param\in\paramdomain}\risksamp(\param; \statval) = 0$ for all $\statval \in \statdomain$.
Suppose $\riskn$ satisfies \Cref{def:interpolation} and has $\lambda$-quadratic growth (respectively $\lambda$-strong convexity). Let $\riskneps$ be defined as in \Cref{lem:growth-minimizer-stability}. Then $\riskneps$ has $\gamma$-quadratic growth (respectively $\gamma$-strong convexity) for any $\gamma \leq \lambda  - \frac{\smooth}{n \varepsilon}$.
\end{lemma}
\begin{proof}
By \Cref{lem:growth-minimizer-stability}, that the minimizing set of $\riskneps$ is the same as $\riskn$.
Suppose for the sake of contradiction that $\riskneps$ does not have $\gamma$-quadratic growth. Then there must exist $\param_1$ such that 
\begin{align*}
    \riskneps(\param_1) - \riskneps(\param\opt) < \frac{\gamma}{2}\ltwo{\param_1 - \param\opt}^2.
\end{align*}
By smoothness and growth we have
\begin{align*}
    \frac{\smooth}{2 n \varepsilon} \ltwo{\param_1 - \param\opt}^2 + \frac{\gamma}{2}\ltwo{\param_1 - \param\opt}^2 > \riskn(\param_1) - \riskn(\param\opt) \geq \frac{\lambda}{2}\ltwo{\param_1 - \param\opt}^2.
\end{align*}
This implies that $\gamma > \lambda - \frac{\smooth}{n\varepsilon}$, a contradiction.

Suppose for the sake of contradiction that $\riskneps$ does not have $\gamma$-strong convexity. Then there must exist $\param_1$ and $\param_2$ such that 
\begin{align*}
    \riskneps(\param_1) - \riskneps(\param_2) < \frac{\gamma}{2}\ltwo{\param_1 - \param\opt}^2 + \langle \nabla \riskneps(\param_2), \param_1 - \param_2 \rangle.
\end{align*}
By smoothness and strong convexity we have
\begin{align*}
    \frac{\smooth}{2 n \varepsilon} \ltwo{\param_1 - \param_2}^2 + \frac{\gamma}{2}\ltwo{\param_1 - \param_2}^2 + \langle \nabla \riskn(\param_2), \param_1 - \param_2 \rangle > \riskn(\param_1) - \riskn(\param_2) \geq \frac{\lambda}{2}\ltwo{\param_1 - \param_2}^2 + \langle \nabla \riskn(\param_2), \param_1 - \param_2 \rangle.
\end{align*}
However, this implies that $\gamma > \lambda - \frac{\smooth}{n\varepsilon}$ which is a contradiction.
\end{proof}

The next lemma is a standard result on the closure under addition of strongly convex functions.
\begin{lemma}\label{lem:sc-closure-add}
Let functions $h_1$ and $h_2$ be $\lambda$ and $\gamma$ strongly convex respectively, then $h_1 + h_2$ is $\lambda + \gamma$ strongly convex.
\end{lemma}

This lemma provides some growth conditions on the gradient under smoothness, strong convexity and  quadratic growth.
\begin{lemma}\label{lem:growth-smooth-consequence}
Let $g: \paramdomain \rightarrow \R_+$ be a convex function with $\paramdomain\opt = \argmin_{\param \in \paramdomain}g(\param)$ such that for $\param\opt \in \paramdomain\opt$, $g(\param\opt) =0$. Suppose $g$ has $\lambda$-quadratic growth, then 
\begin{align*}
    |g'(\param)| \geq \frac{\lambda}{2} \dist(\param, \paramdomain\opt). 
\end{align*}
If instead $g$ has $\lambda$-strong convexity, then
\begin{align*}
    |g'(\param)| \geq \lambda \dist(\param, \paramdomain\opt). 
\end{align*}
Alternatively, suppose $g$ has $\smooth$-smoothness, then
\begin{align*}
    |g'(\param)| \leq \smooth \dist(\param, \paramdomain\opt).
\end{align*}
\end{lemma}
\begin{proof}
We note that by first order optimality conditions, for all $\param\opt \in\paramdomain\opt$, $\nabla g(\param\opt) = 0$.
To prove the first inequality, we have that for any $\param\opt \in \paramdomain\opt$, the following is true:
\begin{align*}
    \frac{\lambda}{2} \dist(\param, \paramdomain\opt)^2 \leq g(\param) - g\opt \leq |g'(\param)||\param - \param\opt|.
\end{align*}
In particular, minimizing over $\param\opt$ on the right hand side and rearranging gives the desired result.
To prove the second result, we know that by strong convexity for any $\param\opt \in \paramdomain\opt$
\begin{align*}
    |g'(\param)|= |g'(\param) - g'(\param\opt)| \geq \lambda |\param - \param\opt|.
\end{align*}
To prove the last result, we know that by smoothness for any $\param\opt \in \paramdomain\opt$
\begin{align*}
    |g'(\param)|= |g'(\param) - g'(\param\opt)| \leq \smooth |\param - \param\opt|.
\end{align*}
Minimizing over $\param\opt$ on the right hand side gives the desired result.
\end{proof}

This lemma controls how much the minimizers of a function can change if another function is added. This will directly be useful in lower bounding the modulus of continuity.
\begin{lemma}\label{lem:control-of-opt-2}
Suppose $h: [-\diam, \diam] \to \R_+$ and $g: [-\diam, \diam] \to \R_+$. Let $\param_h\opt$ be the largest minimizer of $h$ and $\param_g\opt$ be the smallest minimizer of $g$, and assume that $\param_h\opt \leq \param_g\opt$. Let $\param\opt$ be any minimizer of $h +g$. Assume that $h(\param_h\opt) = 0$ and $g(\param_g\opt)= 0$. 
If $h$ has $\lambda_h$-quadratic growth and $g$ is $\smooth_g$-smooth, then 
\begin{align*}
\param\opt -\param_h\opt \leq \frac{\smooth_g (\param_g\opt-\param_h\opt)}{\frac{\lambda_h}{2} + \smooth_g}.
\end{align*}
If $h$ is $\smooth_h$-smooth and $g$ has $\lambda_g$-quadratic growth, then
\begin{align*}
    \frac{\frac{\lambda_g}{2} (\param_g\opt-\param_h\opt)}{\frac{\lambda_g}{2} + \smooth_h} 
    \leq\param\opt -\param_h\opt.
\end{align*}
The same relation holds with $\lambda_g / 2$ and $\lambda_h / 2$ replaced with $\lambda_g$ and $\lambda_h$ respectively if the above statement is modified such that $g$ and $h$ are $\lambda_g$ and $\lambda_h$ strongly convex instead.
\end{lemma}
\begin{proof}
If $\param_h\opt \neq \diam$, then the first order condition for optimality implies
\begin{align*}
    h'(\param_h\opt) + g'(\param_h\opt) = g'(\param_h\opt) < 0 \quad \text{and} \quad     h'(\param_g\opt) + g'(\param_g\opt) = h'(\param_g\opt) > 0.
\end{align*}
Thus, $\param\opt \in (\param_h\opt, \param_g\opt)$. By the monotonicty of the first derivative of convex functions that for $\param\opt \in (\param_h\opt, \param_g\opt)$, $g'(\param\opt) < 0$ and $h'(\param\opt) > 0$. Combining this with \Cref{lem:growth-smooth-consequence}, we get 
\begin{align*}
    \frac{\lambda_h}{2}(\param\opt - \param_h\opt) \leq h'(\param\opt) \leq \smooth_h(\param\opt - \param_h\opt)\\
    \smooth_g (\param\opt - \param_g\opt) \leq g'(\param\opt)\leq \frac{\lambda_g}{2}(\param\opt - \param_g\opt).
\end{align*}
Combining these facts gives
\begin{align*}
\frac{\lambda_h}{2}(\param\opt - \param_h\opt) + \smooth_g (\param\opt -\param_g\opt) \leq h'(\param\opt) + g('\param\opt) = 0 \leq \smooth_h(\param\opt - \param_h\opt) + \frac{\lambda_g}{2} (\param\opt -\param_g\opt)
\end{align*} 
Rearranging these two inequalities gives the desired result. We note that the lower bound only requires that $h$ is $\smooth_h$-smooth and $g$ has $\lambda_g$-quadratic growth, and the upper bound only requires $h$ has $\lambda_h$-quadratic growth and $g$ is $\smooth_g$-smooth. The last statement about strong convexity follows from the same reasoning, except using the strong convexity inequality in \Cref{lem:growth-smooth-consequence} instead of the quadratic growth inequality.
\end{proof}

The following lemma is a slight modification of Claim 6.1 from \cite{ShalevShSrSr09} and will be helpful for us to upper bound the modulus of continuity.
\begin{lemma}\label{lem:stability}
Let $\statvalset'$ consist of $n$ data points where $|\statvalset \triangle \statvalset'| = k$. Suppose that $\riskn$ is $\lambda$-strongly convex and satisfies \Cref{def:interpolation}. Assume the sample function $\risksamp: \paramdomain \times \statdomain \to \R_+$ is $\lipschitz$-Lipschitz in its first argument and that $\inf_{\param\in\paramdomain}\risksamp(\param; \statval) = 0$ for all $\statval \in \statdomain$. For $\paramopts \in \argmin_{\param \in \paramdomain} \riskn(\param)$ and $\paramoptsp \in \argmin_{\param \in \paramdomain} \risknp(\param)$, we have that 
\begin{align*}
    \ltwo{\paramopts - \paramoptsp} \leq \frac{4k\lipschitz}{\lambda n}.
\end{align*}
\end{lemma}
\begin{proof}
By strong convexity, we have that 
\begin{align*}
    \riskn(\paramoptsp) - \riskn(\paramopts) \geq \frac{\lambda}{2} \ltwo{\paramoptsp - \paramopts}^2,
\end{align*}
since by first order optimality conditions, we know that $\nabla \riskn(\paramopts) = 0$ as a consequence of \Cref{def:interpolation}. We also have
\begin{align*}
    \riskn(\paramoptsp) - \riskn(\paramopts) &= \fracnsamp \sum_{\statval \in \statvalset \setminus \statvalset'} \left[ \risksamp(\paramoptsp; \statval) - \risksamp(\paramopts; \statval)\right] + \fracnsamp \sum_{\statval \in \statvalset \cap \statvalset'} \left[ \risksamp(\paramoptsp; \statval) - \risksamp(\paramopts; \statval)\right]\\
    &=\fracnsamp \sum_{\statval \in \statvalset \setminus \statvalset'} \left[ \risksamp(\paramoptsp; \statval) - \risksamp(\paramopts; \statval)\right]
    - \fracnsamp \sum_{\statval \in \statvalset' \setminus \statvalset} \left[ \risksamp(\paramoptsp; \statval) - \risksamp(\paramopts; \statval)\right]
    +\risknp(\paramoptsp) - \risknp(\paramopts)\\
    &\leq \frac{2 k \lipschitz}{n} \ltwo{\paramoptsp - \paramopts},
\end{align*}
where the last inequality comes from the Lipschitzness of $\risksamp$ and that $\paramoptsp \in \argmin_{\param \in \paramdomain} \risknp(\param)$.
\end{proof}

Armed with these supporting lemmas, we can now bound the modulus of continuity. Let $\param_0\opt$ be the largest minimizer of $\riskn$ following the steps of the proof outline.
Without loss of generality, we assume that $\param_0\opt \leq 0$. If $\param_0\opt >0$, by symmetry, it suffices to consider the problem replacing $\frac{\smooth}{2}(\param -\diam)^2$ with $\frac{\smooth}{2}(\param +\diam)^2$ in the following proof.
By \Cref{lem:growth-minimizer-stability}, $\riskneps$ has the same minimizing set as $\riskn$. By \Cref{lem:sc-closure-removal}, $\riskneps$ has $\lambda - \frac{\smooth}{n\varepsilon}$-strong convexity. Replace the $1/\varepsilon$ datapoints removed with samples that have the loss function $\frac{\smooth}{2}(\param-\diam)^2$; we note that it is clear that $\frac{\smooth}{2}(\param-\diam)^2$ satisifies the desired Lipschitz condition. Our constructed non-interpolation population function is
\begin{align*}
    \riskneps(\param) + \frac{\smooth}{2 n\varepsilon}(\param - \diam)^2,
\end{align*}
which is $\lambda$-strongly convex by \Cref{lem:sc-closure-removal} and \Cref{lem:sc-closure-add} and is $2\smooth\diam$-Lipschitz. This means that the $\statvalset'$ this function corresponds to belongs to $\superfamily$. Let $\param\opt$ be the minimizer of $\riskneps(\param) + \frac{\smooth}{2 n\varepsilon}(\param - \diam)^2$.

By the triangle inequality, we have $\riskneps$ is $\left(\frac{n-1/\varepsilon}{n}\right)\smooth$-smooth. $\frac{\smooth}{2n\varepsilon}(\param - \diam)^2$ is $\frac{\smooth}{n\varepsilon}$- strongly convex. Thus, by \Cref{lem:control-of-opt-2} setting $h(x)\defeq \riskneps(x)$ and $g(x)\defeq \frac{\smooth}{2n\varepsilon}(\param - \diam)^2$, we have that 
\begin{align*}
    |\param\opt - \param_0\opt| = \param\opt - \param_0\opt \geq \frac{\frac{\smooth}{n\varepsilon} (\diam - \param_0\opt)}{\left(\frac{n-1/\varepsilon}{n}\right)\smooth + \frac{\smooth}{n\varepsilon}} = \frac{\diam - \param_0\opt}{n \varepsilon} = \frac{\diam + |\param_0\opt|}{n \varepsilon} \geq  \frac{\diam}{n \varepsilon}.
\end{align*}
Here, implicitly, we are using the fact that $\param_0\opt$ is also a minimizer of $\riskneps$ by \Cref{lem:sc-closure-removal}. This completes the proof of the lower bound.

The upper bound follows from \Cref{lem:stability} with $k = 1/\varepsilon$ and $\lipschitz = 2\smooth\diam$.


\end{document}